\documentclass{article}

     \usepackage[final]{neurips_2020}

\usepackage[utf8]{inputenc} 
\usepackage[T1]{fontenc}    
\usepackage{url}            
\usepackage{booktabs}       
\usepackage{amsfonts}       
\usepackage{nicefrac}       
\usepackage{microtype}      
\usepackage{url}

\usepackage{multirow} 
\usepackage[table,dvipsnames]{xcolor}
\usepackage{colortbl}

\usepackage{mathtools}
\usepackage{xcolor}
\newcommand*{\colorboxed}{}
\def\colorboxed#1#{%
	\colorboxedAux{#1}%
}
\newcommand*{\colorboxedAux}[3]{%
	\begingroup
	\colorlet{cb@saved}{.}%
	\color#1{#2}%
	\boxed{%
		\color{cb@saved}%
		#3%
	}%
	\endgroup
}

\usepackage{amsmath,amssymb}
\usepackage{natbib}
\usepackage{algorithm}
\usepackage{algorithmic}
\usepackage{amsthm}

\usepackage{xcolor,colortbl}
\definecolor{DarkBlue}{RGB}{22,54,93}

\usepackage[colorlinks = true]{hyperref}        
\usepackage{cleveref}
\hypersetup{linkcolor =red, citecolor = blue}
\usepackage{subfigure}
\newtheorem{Theorem}{Theorem}
\newtheorem{proposition}{Proposition} 
\newtheorem{lemma}{Lemma} 
\newtheorem{Assumption}{Assumption} 
 
\newtheorem{Definition}{Definition}

\title{Convergence of Meta-Learning with Task-Specific Adaptation over Partial Parameters}

\author{%
  Kaiyi Ji\\
  Department of ECE\\
  The Ohio State University\\
  \texttt{ji.367@osu.edu} \\
   \And
Jason D. Lee \\
  Department of EE\\
  Princeton University \\
   \texttt{jasonlee@princeton.edu} \\
   \AND
  Yingbin Liang \\
  Department of ECE\\
  The Ohio State University\\
   \texttt{liang.889@osu.edu} \\
  \And
  H. Vincent Poor\\
    Department of EE\\
Princeton University \\
  \texttt{poor@princeton.edu} \\
}

\begin{document}

\maketitle

\begin{abstract}
Although model-agnostic meta-learning (MAML) is a very successful algorithm in meta-learning practice, it can have high computational cost because it updates all model parameters over both the inner loop of task-specific adaptation and the outer-loop of meta initialization training. A more efficient algorithm ANIL (which refers to almost no inner loop) was proposed recently by Raghu et al. 2019, which adapts only a small subset of parameters in the inner loop and thus has substantially less computational cost than MAML as demonstrated by extensive experiments. However, the theoretical convergence of ANIL has not been studied yet. In this paper, we characterize the convergence rate and the computational complexity for ANIL under two representative inner-loop loss geometries, i.e., strongly-convexity and nonconvexity. Our results show that such a geometric property can significantly affect the overall convergence performance of ANIL. For example, ANIL achieves a faster convergence rate for a strongly-convex inner-loop loss as the number $N$ of inner-loop gradient descent steps increases, but a slower convergence rate for a nonconvex inner-loop loss as $N$ increases. Moreover, our complexity analysis provides a theoretical quantification on the improved efficiency of ANIL over MAML. The experiments on standard few-shot meta-learning benchmarks validate our theoretical findings. 
\end{abstract}

\section{Introduction} 

As a powerful learning paradigm, meta-learning~\citep{bengio1991learning,thrun2012learning} has recently received significant attention, especially with the incorporation of training deep neural networks~\citep{finn2017model,vinyals2016matching}. Differently from the conventional learning approaches, meta-learning aims to effectively leverage the datasets and prior knowledge of a task ensemble in order to rapidly learn new tasks often with a small amount of data such as in few-shot learning. A broad collection of meta-learning algorithms have been developed so far, which range from metric-based~\citep{koch2015siamese,snell2017prototypical}, model-based~\citep{munkhdalai2017meta,vinyals2016matching}, to optimization-based algorithms~\citep{finn2017model,nichol2018reptile}. The focus of this paper is on the optimization-based approach, which is often easy to be integrated with optimization formulations of many machine learning problems.

One highly successful optimization-based meta-learning approach is the model-agnostic meta-learning (MAML) algorithm~\citep{finn2017model}, which has been applied to many application domains including classification~\citep{rajeswaran2019meta}, reinforcement learning~\citep{finn2017model}, imitation learning~\citep{finn2017one}, etc.  
At a high level, the MAML algorithm consists of two optimization stages: the inner loop of task-specific adaptation and the outer (meta) loop of initialization training. Since the outer loop often adopts a gradient-based algorithm, which takes the gradient over the inner-loop algorithm (i.e., the inner-loop optimization path), even the simple inner loop of gradient descent updating can result in the Hessian update in the outer loop, which causes significant computational and memory cost. Particularly in deep learning, if all neural network parameters are updated in the inner loop, then the cost for the outer loop is extremely high. Thus, designing simplified MAML, especially the inner loop, is highly motivated. ANIL (which stands for {\em almost no inner loop}) proposed in~\cite{raghu2019rapid} has recently arisen as such an appealing approach. In particular,
\cite{raghu2019rapid} proposed to update only a small subset (often only the last layer) of parameters in the inner loop. Extensive experiments  in \cite{raghu2019rapid} demonstrate that ANIL achieves a significant speedup over MAML without sacrificing the performance.

Despite extensive empirical results, there has been no theoretical study of ANIL yet, which motivates this work. In particular, we would like to answer several new questions arising in ANIL (but not in the original MAML). While the outer-loop loss function of ANIL is still nonconvex as MAML, the inner-loop loss can be either  {\em strongly convex} or {\em nonconvex} in practice. The strong convexity  occurs naturally if only the last layer of neural networks is updated in the inner loop, whereas the nonconvexity often occurs if more than one layer of neural networks are updated in the inner loop. Thus, our theory will explore how such different geometries affect the convergence rate, computational complexity, as well as the hyper-parameter selections.  We will also theoretically quantify how much computational advantage ANIL achieves over MAML by training only partial parameters in the inner loop. 

%

\subsection{Summary of Contributions}\label{sec:contribution}

In this paper, we characterize the convergence rate and the computational complexity for ANIL with $N$-step inner-loop gradient descent, under nonconvex outer-loop loss geometry, and under two representative inner-loop loss geometries, i.e., strongly-convexity and nonconvexity. Our analysis also provides theoretical guidelines for choosing the hyper-parameters such as the stepsize and the number $N$ of inner-loop steps under each geometry. We summarize our specific results as follows.

\begin{list}{$\bullet$}{\topsep=0.2ex \leftmargin=0.26in \rightmargin=0.in \itemsep =0.05in}

\item {\bf Convergence rate}: ANIL converges sublinearly with the convergence error decaying sublinearly with the number of sampled tasks due to nonconvexity of the meta objective function. The convergence rate is further significantly affected by the geometry of the inner loop. Specifically, ANIL converges exponentially fast with $N$ initially and then saturates under the strongly-convex inner loop, and constantly converges slower as $N$ increases under the nonconvex inner loop.

\item {\bf Computational complexity}: ANIL attains an $\epsilon$-accurate stationary point with the gradient and second-order evaluations at the order of $\mathcal{O}(\epsilon^{-2})$ due to nonconvexity of the meta objective function. The computational cost is also significantly affected by the geometry of the inner loop. Specifically, under the strongly-convex inner loop, its complexity first decreases and then increases with $N$, which suggests a moderate value of $N$ and a constant stepsize in practice for a fast training. But under the nonconvex inner loop, ANIL has higher computational cost as $N$ increases, which suggests a small $N$ and a stepsize at the level of $1/N$ for desirable training.




\item Our experiments validate that ANIL exhibits aforementioned {\em very different} convergence behaviors under the two inner-loop geometries.



\end{list}
From the technical standpoint, we develop new techniques to capture the properties for ANIL, which does not follow from the existing theory for MAML~\citep{fallah2019convergence,ji2020multi}. First, our analysis explores how different geometries of the inner-loop loss (i.e., strongly-convexity and nonconvexity) affect the convergence of ANIL. Such comparison does not exist in MAML. Second, ANIL contains parameters that are updated only in the outer loop, which exhibit {\em special} meta-gradient properties not captured in MAML.


\subsection{Related Works}
{\bf MAML-type meta-learning approaches.} 
As a pioneering meta-initialization approach, MAML~\citep{finn2017model} aims to find a good  initialization point such that a few gradient descent steps starting from this point achieves fast adaptation. 
MAML has inspired various variant algorithms 
~\citep{finn2017meta,finn2019online,finn2018probabilistic,jerfel2018online,mi2019meta,raghu2019rapid,rajeswaran2019meta,zhou2019efficient}. 
 For example, FOMAML~\citep{finn2017model} and Reptile~\citep{nichol2018reptile}  are two first-order MAML-type algorithms which avoid second-order derivatives. \cite{finn2019online} provided an extension of MAML to the online setting.
  Based on the implicit differentiation technique,~\cite{rajeswaran2019meta} proposed a MAML variant named iMAML by formulating the inner loop as a regularized empirical risk minimization problem.
  More recently,~\cite{raghu2019rapid} modifies MAML to ANIL by adapting a small subset of model parameters during the inner loop in order to reduce the computational and memory cost. This paper provides the theoretical  guarantee for ANIL as a complement to its empirical study in~\cite{raghu2019rapid}.

\vspace{0.1cm}
{\noindent \bf Other optimization-based meta-learning approaches.} Apart from MAML-type meta-initialization algorithms, another well-established framework in few-shot meta learning~\citep{bertinetto2018meta,lee2019meta,ravi2016optimization,snell2017prototypical,zhou2018deep} aims to learn good parameters as a common embedding model for all tasks. 
Building on the embedded features, task-specific parameters are then searched as a minimizer of the inner-loop loss function~\citep{bertinetto2018meta,lee2019meta}. Compared to ANIL, such a framework does not train the task-specific parameters as initialization, whereas ANIL trains a good initialization for the task-specific parameters.



\vspace{0.1cm}
 \noindent{\bf Theory for MAML-type approaches.} 
 There have been only a few studies on the statistical and convergence performance of MAML-type algorithms. 
 \cite{finn2017meta} proved a universal approximation property of MAML under mild conditions. 
\cite{rajeswaran2019meta} analyzed the convergence of iMAML algorithm based on implicit meta gradients. 
 \cite{fallah2019convergence} analyzed the convergence of one-step MAML for a nonconvex objective, and~\cite{ji2020multi} analyzed the convergence of multi-step MAML in the nonconvex setting.  
 As a comparison, we analyze the ANIL algorithm provided in~\cite{raghu2019rapid}, which has different properties from MAML due to adapting only partial parameters in the inner loop.
 
 
 \vspace{0.1cm}
  {\noindent\bf Notations.} For a function {\small $L(w,\phi)$} and a realization {\small $( w^\prime, \phi^\prime)$}, we define {\small $\nabla_{w} L (w^\prime,\phi^\prime) =\frac{\partial L (w,\phi)}{\partial w}\big |_{(w^\prime, \phi^\prime)} $, $\nabla^2_{w} L( w^\prime, \phi^\prime)=\frac{\partial^2 L(w,\phi)}{\partial w^2}\big |_{(w^\prime, \phi^\prime)}$,$\nabla_\phi\nabla_w L(w^\prime, \phi^\prime) = \frac{\partial^2 L(w,\phi)}{\partial \phi \partial w}\big |_{( w^\prime, \phi^\prime)}$}. The same notations hold for $\phi$.

\section{Problem Formulation and Algorithms}\label{problem}

Let $\mathcal{T}=(\mathcal{T}_i, i\in \mathcal{I})$ be a set of tasks available for meta-learning, where tasks are sampled for use by a distribution of $p_\mathcal{T}$. Each task $\mathcal{T}_i$ contains a training sample set $\mathcal{S}_i$ and a test set $\mathcal{D}_i$. Suppose that meta-learning divides all model parameters into mutually-exclusive sets $(w,\phi)$ as described below.
\begin{list}{$\bullet$}{\topsep=0.8ex \leftmargin=0.2in \rightmargin=0.in \itemsep =0.05in}
\item $w$ includes task-specific parameters, and meta-learning trains a good initialization of $w$. 

\item $\phi$ includes common parameters shared by all tasks, and meta-learning trains $\phi$ for direct reuse. 
\end{list}
For example, in training neural networks, $w$ often represents the parameters of some partial layers,  and $\phi$ represents the parameters of the remaining inner layers.  The goal of meta-learning here is to jointly learn $w$ as a good initialization parameter and $\phi$ as a reuse parameter, such that $(w_N,\phi)$ performs well on a sampled individual task $\mathcal{T}$, where $w_N$ is the $N$-step gradient descent update of $w$. 
To this end, ANIL solves the following optimization problem with the objective function given by
\begin{align}\label{eq:obj} 
\hspace{-0.3cm}\text{(Meta objective function):}\quad & 
\min_{ w,\phi} L^{meta} ( w, \phi) := \mathbb{E}_{i\sim p_{\mathcal{T}}} L_{\mathcal{D}_i}( w^i_N( w,\phi), \phi), 
\end{align}
where the loss function $L_{\mathcal{D}_i}(w^i_N, \phi):=\sum_{\xi \in \mathcal{D}_i } \ell (w^i_N,\phi; \xi)$ takes the finite-sum form over the test dataset $\mathcal{D}_i$, and the parameter $w^i_N$  for task $i$ is obtained via an inner-loop $N$-step gradient descent update of $w^i_0=w$ (aiming to minimize the task $i$'s loss function $L_{\mathcal{S}_i}(w,\phi)$ over $w$) as given by
\begin{align}\label{inner:gd}
\text{(Inner-loop gradient descent):} \quad w_{m+1}^i =  w_{m}^i - \alpha \nabla_{w} L_{\mathcal{S}_i} (w_{m}^i,\phi),\, m=0,1,...,N-1. 
\end{align}
Here, $w^i_N(w,\phi)$ explicitly indicates the dependence of $w^i_N$ on $\phi$ and the initialization $w$ via the iterative updates in~\cref{inner:gd}. To draw connection, the problem here reduces to the  MAML~\citep{finn2017model} framework if $w$ includes all training parameters and $\phi$ is empty, i.e., no parameters are reused directly.
\subsection{ANIL Algorithm}\label{Algorithm}
ANIL~\citep{raghu2019rapid} (as described in \Cref{alg:anil}) solves the problem in \cref{eq:obj} via two nested optimization loops, i.e., inner loop for task-specific adaptation and outer loop for updating meta-initialization and reuse parameters. At the $k$-th outer loop, ANIL samples a batch $\mathcal{B}_k$ of identical and independently distributed (i.i.d.) tasks based on $p_{\mathcal{T}}$. Then, each task in $\mathcal{B}_k$ runs an inner loop of $N$ steps of gradient descent with a stepsize $\alpha$ as in lines $5$-$7$ in~\Cref{alg:anil}, where $w_{k,0}^i =  w_k$ for all tasks $\mathcal{T}_i\in\mathcal{B}_{k}$.

After obtaining the inner-loop output $w^i_{k,N}$ for all tasks, ANIL computes two partial gradients $\frac{\partial L_{\mathcal{D}_i}(w^i_{k,N},\,\phi_k)}{\partial { w_k}}$ and $\frac{\partial L_{\mathcal{D}_i}(w^i_{k,N},\,\phi_k)}{\partial {\phi_k}}$ 
respectively by back-propagation, and updates $w_k$ and $\phi_k$ by stochastic gradient descent as in line $10$ in~\Cref{alg:anil}.
Note that $\phi_k$ and $w_k$ are treated to be mutually-independent during the differentiation process. Due to the nested dependence of $w_{k,N}^i$ on $\phi_k$ and $w_k$, the two partial gradients involve complicated second-order derivatives. Their explicit forms are provided in the following proposition. 
\begin{proposition}\label{le:gd_form}
The partial meta gradients take the following explicit form: 
\begin{small} 
\begin{align*}
{1)} \quad \frac{\partial L_{\mathcal{D}_i}( w^i_{k,N}, \phi_k)}{\partial w_k} =& \prod_{m=0}^{N-1}(I - \alpha \nabla_w^2L_{\mathcal{S}_i}(w_{k,m}^i,\phi_k)) \nabla_{w} L_{\mathcal{D}_i} (w_{k,N}^i,\phi_k). \nonumber
\\ {2)}\quad \frac{\partial L_{\mathcal{D}_i}( w^i_{k,N}, \phi_k)}{\partial \phi_k} =& -\alpha \sum_{m=0}^{N-1}\nabla_\phi\nabla_w L_{\mathcal{S}_i}(w_{k,m}^i,\phi_k) \prod_{j=m+1}^{N-1}(I-\alpha\nabla_w^2L_{\mathcal{S}_i}(w_{k,j}^i,\phi_k))\nabla_w L_{\mathcal{D}_i}(w_{k,N}^i,\phi_k) \nonumber
\\&+\nabla_\phi L_{\mathcal{D}_i}(w_{k,N}^i,\phi_k).
\end{align*}
\end{small}
\end{proposition}

 \begin{algorithm}[t]
	\caption{ANIL Algorithm} 
	\label{alg:anil}
	\begin{algorithmic}[1]
		\STATE {\bfseries Input:} Distribution over tasks $p_\mathcal{T}$, inner stepsize $\alpha$, outer stepsize $\beta_w,\beta_\phi$, initialization $ w_0, \phi_0$ 
		\WHILE{not converged}
		\STATE{Sample a mini-batch of i.i.d. tasks $\mathcal{B}_k= \{\mathcal{T}_i\}_{i=1}^B$  based on  the distribution $p_\mathcal{T}$}
		\FOR{each task $\mathcal{T}_i$ in $\mathcal{B}_k$}
		\FOR{$m=0,1,...,N-1$}
		\STATE{Update $w_{k,m+1}^i = w_{k,m}^i - \alpha \nabla_w L_{\mathcal{S}_i}( w_{k,m}^i,\phi_k)$}
		\ENDFOR
		\STATE{Compute gradients $\frac{\partial L_{\mathcal{D}_i}(w^i_{k,N},\phi_k)}{\partial { w_k}},\frac{\partial L_{\mathcal{D}_i}(w^i_{k,N},\phi_k)}{\partial {\phi_k}}$ by back-propagation}
		\ENDFOR
                 \STATE{Update parameters $ w_k$ and $\phi_k$ by mini-batch SGD:
                 \begin{align*}
                 w_{k+1}=  w_{k} - \frac{\beta_w}{B}\sum_{i\in\mathcal{B}_k} \frac{\partial L_{\mathcal{D}_i}( w^i_{k,N},\phi_k)}{\partial {w_k}}, \quad \phi_{k+1}= \phi_{k} - \frac{\beta_\phi}{B}\sum_{i\in\mathcal{B}_k} \frac{\partial L_{\mathcal{D}_i}( w^i_{k,N},\phi_k)}{\partial {\phi_k}}
                 \end{align*}
                 \vspace{-0.2cm}
                 }
                                  \STATE{Update $k \leftarrow k+1$}
		\ENDWHILE
	\end{algorithmic}
	\end{algorithm}
\subsection{Technical Assumptions and Definitions}\label{se:pre} 
We let $z=(w,\phi)\in\mathbb{R}^{n}$ denote all parameters. For simplicity, suppose $\mathcal{S}_i$ and $\mathcal{D}_i$ for all $i\in\mathcal{I}$ have sizes of $S$ and $D$, respectively. In this paper, we consider the following types of loss functions.
\begin{list}{$\bullet$}{\topsep=0.7ex \leftmargin=0.2in \rightmargin=0.in \itemsep =0.05in}

\item The outer-loop meta loss function in \cref{eq:obj} takes the finite-sum form as $L_{\mathcal{D}_i}(w^i_N, \phi):=\sum_{\xi \in \mathcal{D}_i } \ell (w^i_N,\phi; \xi)$. It is generally nonconvex in terms of both $w$ and $\phi$.

\item The inner-loop loss function $L_{\mathcal{S}_i}(w,\phi)$ with respect to $w$ has two cases: strongly-convexity and nonconvexity. The strongly-convex case occurs often when $w$ corresponds to parameters of the last {\em linear} layer of a neural network, so that the loss function of such a $w$ is naturally chosen to be a quadratic function or a logistic loss with a strongly convex regularizer~\citep{bertinetto2018meta,lee2019meta}. The nonconvex case can occur if $w$ represents parameters of more than one layers (e.g., last two layers~\citep{raghu2019rapid}). As we prove in \Cref{sec:mainresult}, such geometries affect the convergence rate significantly.
\vspace{0.1cm}
\end{list}
Since the objective function $L^{meta}(w,\phi)$ in~\cref{eq:obj} is generally nonconvex, we use the gradient norm as the convergence criterion, which is standard in nonconvex optimization.
\begin{Definition}
We say that $(\bar{w},\bar{\phi})$ is an $\epsilon$-accurate solution for the meta optimization problem in~\cref{eq:obj} if 
{\small $\,\mathbb{E}\Big\| \frac{\partial L^{meta}(\bar{w},\bar{\phi})}{\partial \bar w}\Big\|^2 <\epsilon$ and $\mathbb{E}\Big\| \frac{\partial L^{meta}(\bar{w},\bar{\phi})}{\partial \bar \phi}\Big\|^2 <\epsilon$}.
\end{Definition}
We further take the following standard assumptions on the {\em individual} loss function for each task, which have been commonly adopted in conventional minimization problems~\citep{ghadimi2013stochastic,ji2019history,ji2019improved,wang2018spiderboost} and min-max optimization~\citep{lin2020near} as well as the MAML-type optimization~\citep{finn2019online,ji2020multi}. 
\begin{Assumption}\label{assm:smooth}
The loss function $L_{\mathcal{S}_i}(z)$ and $L_{\mathcal{D}_i}(z)$ for each task $\mathcal{T}_i$ satisfy:
\begin{list}{$\bullet$}{\topsep=0.7ex \leftmargin=0.3in \rightmargin=0.5in \itemsep =0.05in}
\vspace{0.2cm}
\item $L_{\mathcal{S}_i}(z)$ and $L_{\mathcal{D}_i}(z)$ are $L$-smooth, i.e., for any $z,z^\prime \in \mathbb{R}^{n}$, $$\|\nabla L_{\mathcal{S}_i}(z) - \nabla L_{\mathcal{S}_i}(z^\prime) \| \leq L\|z-z^\prime\|, \|\nabla L_{\mathcal{D}_i}(z) - \nabla L_{\mathcal{D}_i}(z^\prime) \| \leq L\|z-z^\prime\|.$$
\item $L_{\mathcal{D}_i}(z)$ is $M$-Lipschitz, i.e., for any $z,z^\prime\in \mathbb{R}^{n}$,  $| L_{\mathcal{D}_i}(z) -  L_{\mathcal{D}_i}(z^\prime) | \leq M\|z-z^\prime\|$.
\end{list}
\end{Assumption}


Note that we {\em do not} impose the function Lipschitz assumption (i.e., item 2 in~\Cref{assm:smooth}) on the inner-loop loss function $L_{S_i}(z)$. We take the assumption on the Lipschitzness of function $L_{\mathcal{D}_i}$ to ensure the meta gradient to be bounded. We note that  iMAML~\citep{rajeswaran2019meta} alternatively assumes the search space of parameters to be bounded (see Theorem 1 therein) so that the meta gradient (eq. (5) therein) can be bounded.

As shown in~\Cref{le:gd_form},  the partial meta gradients involve two types of high-order derivatives {\small$\nabla_w^2 L_{\mathcal{S}_i}(\cdot,\cdot)$} and {\small$\nabla_\phi\nabla_w L_{\mathcal{S}_i}(\cdot,\cdot)$}. 
The following assumption imposes a Lipschitz condition for these two high-order derivatives, which has been widely adopted in optimization problems that involve two sets of parameters, e.g, bi-level programming~\citep{ghadimi2018approximation}.
\begin{Assumption}\label{assm:second} Both $\nabla_w^2 L_{\mathcal{S}_i}(z)$ and $\nabla_\phi\nabla_w L_{\mathcal{S}_i}(z)$ are $\rho$-Lipschitz and $\tau$-Lipschitz, i.e., 
\begin{list}{$\bullet$}{\topsep=0.7ex \leftmargin=0.3in \rightmargin=0.5in \itemsep =0.05in}
\vspace{0.2cm}
\item For any $z,z^\prime\in\mathbb{R}^{n}$, $\|\nabla^2_w L_{\mathcal{S}_i}(z) - \nabla^2_w L_{\mathcal{S}_i}(z^\prime)\| \leq \rho \|z - z^\prime\|$.
\item For any $z,z^\prime\in\mathbb{R}^{n}$, $\|\nabla_\phi \nabla_w L_{\mathcal{S}_i}(z) - \nabla_\phi\nabla_w L_{\mathcal{S}_i}(z^\prime)\| \leq \tau\|z-z^\prime\|$.
\end{list}
\end{Assumption}
\section{Convergence Analysis of ANIL}\label{sec:mainresult}
We first provide convergence analysis for the ANIL algorithm, and then compare the performance of ANIL under two geometries and compare the performance between ANIL and MAML.

\subsection{Convergence Analysis under Strongly-Convex Inner-Loop Geometry}\label{se:strong-convex}
We first analyze the convergence rate of ANIL for the case where the inner-loop loss function $L_{\mathcal{S}_i}(\cdot,\phi)$ satisfies the following strongly-convex condition. 
\begin{Definition}
$L_{\mathcal{S}_i}(w,\phi)$ is $\mu$-strongly convex with respect to $w$ if for any $w,w^\prime$ and $\phi$,
 \begin{align*}
 L_{\mathcal{S}_i}(w^\prime,\phi) \geq L_{\mathcal{S}_i}(w,\phi) + \big\langle w^\prime -w , \nabla_w L_{\mathcal{S}_i}(w,\phi) \big\rangle + \frac{\mu}{2} \|w-w^\prime\|^2.
 \end{align*} 
 \end{Definition}
  \vspace{-0.2cm}

Based on  \Cref{le:gd_form}, we characterize the smoothness property of $L^{meta}(w,\phi)$ in~\cref{eq:obj} as below. 
\begin{proposition}\label{le:strong-convex}
Suppose Assumptions~\ref{assm:smooth} and \ref{assm:second} hold
and choose the inner stepsize $\alpha=\frac{\mu}{L^2}$. Then, for any two points $(w_1,\phi_1), (w_2,\phi_2)\in\mathbb{R}^n$, we have
\begin{small}
\begin{align*}
1)  \quad \Big\| &\frac{\partial L^{meta}( w,\phi)}{\partial w} \Big |_{(w_1,\phi_1)} - \frac{\partial L^{meta}( w,\phi)}{\partial w} \Big |_{(w_2,\phi_2)} \Big\| 
\\&\leq \text{\normalfont poly}(L,M,\rho)\frac{L}{\mu}(1-\alpha \mu)^{N}\|w_1-w_2\|+\text{\normalfont poly}(L,M,\rho)\left(\frac{L}{\mu}+1\right)N(1-\alpha \mu)^{N}\|\phi_1-\phi_2\|,
\\ 2)\quad  \Big\| &\frac{\partial L^{meta}( w,\phi)}{\partial \phi} \Big |_{(w_1,\phi_1)} -  \frac{\partial L^{meta}( w,\phi)}{\partial \phi} \Big |_{(w_2,\phi_2)} \Big\|
\\&\leq \text{\normalfont poly}(L,M,\tau,\rho) \frac{L}{\mu}(1-\alpha\mu)^{\frac{N}{2}}\|w_1-w_2\|+\text{\normalfont poly}(L,M,\rho) \frac{L^3}{\mu^3}\|\phi_1-\phi_2\|,
\end{align*}
\end{small}
\hspace{-0.12cm} 
where $\tau,\rho,L$ and $M$ are given in Assumptions~\ref{assm:smooth} and \ref{assm:second}, and $\text{\normalfont poly}(\cdot)$ denotes the polynomial function of the parameters with the explicit forms 
given in~\Cref{append:str}.
\end{proposition}


 \Cref{le:strong-convex} indicates that increasing the number $N$ of inner-loop gradient descent steps yields much {\em smaller} smoothness parameters for the meta objective function $L^{meta}(w,\phi)$. As shown in the following theorem, this allows a larger stepsize $\beta_w$, which yields a faster convergence rate $\mathcal{O}(\frac{1}{K\beta_w})$.
 
%

\begin{Theorem}\label{th:strong-convex} 
Suppose Assumptions~\ref{assm:smooth} and \ref{assm:second} hold, and apply \Cref{alg:anil} to solve the meta optimization problem~\cref{eq:obj} with stepsizes $\beta_w={\small\text{\normalfont poly}(\rho,\tau,L,M) \mu^2(1-\frac{\mu^2}{L^2})^{-\frac{N}{2}}}$ and $\beta_\phi={\small\text{\normalfont poly}(\rho,\tau,L,M)\mu^{3}}$. 
Then, ANIL  finds a point $(w,\phi)\in\big\{(w_k,\phi_k),k=0,...,K-1\big\}$ such that 
\begin{align*}
\text{\normalfont (Rate w.r.t. $w$)}\quad\mathbb{E}\left\| \frac{\partial L^{meta}(w,\phi)}{\partial w}\right\|^2  \leq& \mathcal{O}\Bigg( \frac{  \frac{1}{\mu^2}\left(1-\frac{\mu^2}{L^2}\right)^{\frac{N}{2}}}{K}      +\frac{\frac{1}{\mu} \left(1-\frac{\mu^2}{L^2}\right)^{\frac{N}{2}}}{B}     \Bigg), \nonumber
\\\text{\normalfont (Rate w.r.t. $\phi$)}\quad \mathbb{E}\left\| \frac{\partial L^{meta}(w,\phi)}{\partial \phi}\right\|^2  \leq & \mathcal{O}\Bigg(\frac{ \frac{1}{\mu^2}\left(1-\frac{\mu^2}{L^2}\right)^{\frac{N}{2}}+\frac{1}{\mu^3}}{K} +\frac{\frac{1}{\mu}\left(1-\frac{\mu^2}{L^2}\right)^{\frac{3N}{2}}+\frac{1}{\mu^2}}{B}\Bigg).
\end{align*}
To achieve an $\epsilon$-accurate point, ANIL requires at most {\small $\mathcal{O}\big(\frac{c_wN}{\mu^{4}}\big(1-\frac{\mu^2}{L^2}\big)^{N/2}+\frac{c_w^\prime N}{\mu^{5}}\big)\epsilon^{-2}$} 
 gradient evaluations in $w$,  {\small $\mathcal{O}\big(\frac{c_\phi}{\mu^{4}}\big(1-\frac{\mu^2}{L^2}\big)^{N/2}+\frac{c_\phi^\prime}{\mu^{5}}\big)\epsilon^{-2}$} gradient evaluations in $\phi$,  and {\small $\mathcal{O}\big(\frac{c_sN}{\mu^{4}}\big(1-\frac{\mu^2}{L^2}\big)^{N/2}+\frac{c_s^\prime N}{\mu^{5}}\big)\epsilon^{-2}$} second-order derivative evaluations of {\small$\nabla_w^2 L_{\mathcal{S}_i}(\cdot,\cdot)$} and {\small$\nabla_\phi\nabla_w L_{\mathcal{S}_i}(\cdot,\cdot)$}, where constants $c_w,c_w^\prime,c_\phi,c_\phi^\prime,c_s,c_s^\prime$ depend on $\tau,M,\rho$. 
\end{Theorem}


\Cref{th:strong-convex} shows that ANIL converges sublinearly with the number $K$ of outer-loop meta iterations, and the convergence error decays sublinearly with the number $B$ of sampled tasks, which are consistent with the nonconvex nature of the meta objective function. The convergence rate is further significantly affected by the number $N$ of the inner-loop steps. Specifically, with respect to $w$, ANIL converges exponentially fast as $N$ increases due to the strong convexity of the inner-loop loss. With respect to $\phi$, the convergence rate depends on two components: an exponential decay term with $N$ and an $N$-independent term. As a result, the overall convergence of meta optimization becomes faster as $N$ increases, and then saturates for large enough $N$ as the second component starts to dominate. 
This is demonstrated by our experiments in \Cref{exp:strongly-convex}.

\Cref{th:strong-convex} further indicates that ANIL attains an $\epsilon$-accurate stationary point with the gradient and second-order evaluations at the order of $\mathcal{O}(\epsilon^{-2})$ due to nonconvexity of the meta objective function. The computational cost is further significantly affected by inner-loop steps. Specifically, the gradient and second-order derivative evaluations contain two terms: an exponential decay term with $N$ and a linear growth term with $N$. 
For a large condition number $\kappa$,  a small $N$, e.g., $N=2$, is a better choice. However, when $\kappa$ is not very large, e.g., in our experiments in \Cref{exp:strongly-convex} (in which increasing $N$ accelerates the iteration rate), 
 the computational cost of ANIL initially decreases because the exponential reduction dominates the linear growth. But when $N$ is large enough, the exponential decay saturates and the linear growth dominates, and hence the overall computational cost of ANIL gets higher as $N$ further increases. 
 This suggests to take a moderate but not too large $N$ in practice to achieve an optimized performance, which we also demonstrate in our experiments in \Cref{exp:strongly-convex}.

\subsection{Convergence Analysis under Nonconvex Inner-Loop Geometry}\label{sec:nonconvex}
In this subsection, we study the case, in which the inner-loop loss function $L_{\mathcal{S}_i}(\cdot,\phi)$ is  nonconvex. 
The following proposition characterizes the smoothness of $L^{meta}(w,\phi)$ in~\cref{eq:obj}.

\begin{proposition}\label{le:smooth_nonconvex}
Suppose Assumptions~\ref{assm:smooth} and \ref{assm:second} hold, and choose the inner-loop stepsize {\small$\alpha <\mathcal{O}(\frac{1}{N})$}.  
Then, for any two points $(w_1,\phi_1)$, $(w_2,\phi_2)\in\mathbb{R}^n$, we have 
\vspace{0.1cm}
\begin{small}
\begin{align*}
1) \bigg\| &\frac{\partial L^{meta}( w,\phi)}{\partial w} \Big |_{(w_1,\phi_1)} -  \frac{\partial L^{meta}( w,\phi)}{\partial w} \Big |_{(w_2,\phi_2)} \bigg\| \leq\text{\normalfont poly}(M,\rho,\alpha,L) N(\|w_1-w_2\| + \|\phi_1-\phi_2\|),
\\ 2) \bigg\| &\frac{\partial L^{meta}( w,\phi)}{\partial \phi} \Big |_{(w_1,\phi_1)} -  \frac{\partial L^{meta}( w,\phi)}{\partial \phi} \Big |_{(w_2,\phi_2)} \bigg\|\leq \text{\normalfont poly}(M,\rho,\tau,\alpha,L) N(\|w_1-w_2\| + \|\phi_1-\phi_2\|),
\end{align*}
\end{small}
\vspace{0.1cm}
\hspace{-0.12cm} 
where $\tau,\rho,L$ and $M$ are given by Assumptions~\ref{assm:smooth} and \ref{assm:second}, and $\text{\normalfont poly}(\cdot)$ denotes the polynomial function of the parameters with the explicit forms of the smoothness parameters given in~\Cref{appen:smooth_nonconvex}.

\end{proposition}
\Cref{le:smooth_nonconvex} indicates that the meta objective function $L^{meta}(w,\phi)$ is smooth with respect to both $w$ and $\phi$ with their smoothness parameters increasing linearly with $N$. Hence, $N$ should be chosen to be small so that the outer-loop meta optimization can take reasonably large stepsize to run fast. Such a property is in sharp contrast to the strongly-convex case in which the corresponding smoothness parameters decrease with $N$.


 The following theorem provides the convergence rate of ANIL under the nonconvex inner-loop loss.
\begin{Theorem}\label{th:nonconvex} 
Under the setting of~\Cref{le:smooth_nonconvex}, and apply \Cref{alg:anil} to solve the meta optimization problem in~\cref{eq:obj} with the stepsizes {\small $\beta_w = \beta_\phi=\text{\normalfont poly}(\rho,\tau,M,\alpha,L) N^{-1}$}. 
Then, ANIL  finds a point $(w,\phi)\in\{(w_k,\phi_k),k=0,...,K-1\}$ such that 
\vspace{0.1cm}
\begin{align*}
\mathbb{E}\left\| \frac{\partial L^{meta}(w,\phi)}{\partial w}\right\|^2  \leq& \,\mathcal{O}\bigg(  \frac{N}{K} + \frac{N}{B}  \bigg), \qquad
\mathbb{E}\left\| \frac{\partial L^{meta}(w,\phi)}{\partial \phi}\right\|^2  \leq \,\mathcal{O}\bigg(  \frac{N}{K} + \frac{N}{B}  \bigg).
\end{align*}
\vspace{0.1cm}
To achieve an $\epsilon$-accurate point, ANIL requires at most {\small $\mathcal{O}(N^2\epsilon^{-2})$} gradient evaluations in $w$, {\small $\mathcal{O}(N\epsilon^{-2})$} gradient evaluations in $\phi$,  and {\small $\mathcal{O}(N^2\epsilon^{-2})$} second-order derivative evaluations.
\end{Theorem}
\Cref{th:nonconvex} shows that ANIL converges sublinearly with $K$, the convergence error decays sublinearly with $B$, and the computational complexity scales at the order of $\mathcal{O}(\epsilon^{-2})$. But the nonconvexity of the inner loop affects the convergence very differently. Specifically, increasing the number $N$ of the inner-loop gradient descent steps yields slower convergence and higher computational complexity. This suggests to choose a relatively small $N$ for an efficient optimization process, which is demonstrated in our experiments in \Cref{exp:nonconvex}
\subsection{Complexity Comparison of Different Geometries and Different Algorithms}

In this subsection, we first compare the performance for ANIL under strongly convex and nonconvex inner-loop loss functions, and then compare the performance between ANIL and MAML.

\begin{table*}[th] 
	\centering 
	\small
	\caption{Comparison of different geometries on the convergence rate and complexity of ANIL.}
	\vspace{0.2cm}
	\begin{tabular}{lccc} \toprule
		{Geometries} &Convergence rate  & Gradient complexity  & Second-order complexity  \\   \midrule
		Strongly convex & $\mathcal{O}\Big(\frac{ (1-\xi)^{\frac{N}{2}}+c_k}{K} +\frac{(1-\xi)^{\frac{3N}{2}}+c_b}{B}\Big){\color{red}^\sharp}$ &    $\mathcal{O}\Big(\frac{N((1-\xi)^{\frac{N}{2}}+c_\epsilon)}{\epsilon^2}\Big){\color{red}^\S}$ &$\mathcal{O}\Big(\frac{N((1-\xi)^{\frac{N}{2}}+c_\epsilon)}{\epsilon^2}\Big)$\\   
		Nonconvex &$\mathcal{O}\Big(  \frac{N}{K} + \frac{N}{B}  \Big)$ &    $\mathcal{O}\big(\frac{N^2}{\epsilon^{2}}\big)$ & $\mathcal{O}\big(\frac{N^2}{\epsilon^{2}}\big)$  \\  
		\bottomrule	
		\multicolumn{4}{l}{%
  \begin{minipage}{11cm}%
  \vspace{0.1cm}
    \footnotesize Each order term in the table summarizes the dominant components of both $w$ and $\phi$. \\${\color{red}^\sharp}:$ $\xi=\frac{\mu^2}{L^2}<1$, $c_k,c_b$ are constants.  ${\color{red}^\S}:$ $c_\epsilon$ is constant.
  \end{minipage}%
}\\
	\end{tabular} 
	\label{tas11}
\end{table*}

{\noindent \bf Comparison for ANIL between strongly convex and nonconvex inner-loop  geometries:} 
Our results in \Cref{se:strong-convex,sec:nonconvex} have showed that the inner-loop geometry can significantly affect the convergence rate and the computational complexity of ANIL. The detailed comparison is provided in~\Cref{tas11}. It can be seen that increasing $N$ yields a faster convergence rate  for the strongly-convex inner loop, but a slower convergence rate for the nonconvex inner loop. \Cref{tas11} also indicates that increasing $N$ first reduces and then increases the computational complexity for the strongly-convex inner loop, but constantly increases the complexity for the nonconvex inner loop. 


We next provide an intuitive explanation for such different behaviors under these two geometries. For the nonconvex inner loop, $N$ gradient descent iterations starting from two different initializations  likely reach two points that are far away from each other due to the nonconvex landscape so that  the meta objective function can have a large smoothness parameter. Consequently, the stepsize should be small to avoid divergence, which yields slow convergence.
However, for the strongly-convex inner loop, also consider two $N$-step inner-loop gradient descent paths. Due to the strong convexity, they both approach to the same unique optimal point, and hence their corresponding values of the meta objective function are guaranteed to be close to each other as $N$ increases. Thus, increasing $N$ reduces the smoothness parameter, and allows a faster convergence rate. 


\vspace{0.2cm}
{\noindent \bf Comparison between ANIL and MAML:} \cite{raghu2019rapid} empirically showed that ANIL significantly speeds up MAML due to the fact that only a very small subset of parameters go through the inner-loop update. The complexity results in~\Cref{th:strong-convex} and \Cref{th:nonconvex} provide theoretical characterization of such an acceleration. To formally compare the performance between ANIL and MAML, let $n_w$ and $n_\phi$ be the dimensions of $w$ and $\phi$, respectively. 
The detailed comparison is provided in~\Cref{table:maml_anil}. 

 

\begin{table*}[th] 
	\centering 
	\small
	\caption{Comparison of the computational complexities of ANIL and MAML.}
	\vspace{0.1cm}
	\begin{tabular}{lcc} \toprule
		{Algorithms} & \# of gradient entry evaluations ${\color{red}^\sharp }$  & \# of second-order  entry evaluations${\color{red}^\S }$  		 \\   \midrule
		MAML~\cite[Theorem 2]{ji2020multi} & $\mathcal{O}\Big(\frac{(Nn_w+Nn_\phi)(1+\kappa L)^N}{\epsilon^{2}}\Big){\color{red}^\aleph }$&    $\mathcal{O}\Big(\frac{(n_w+n_\phi)^2N(1+\kappa L)^N}{\epsilon^{2}}\Big)$	 \\  
		ANIL (Strongly convex)& $\mathcal{O}\Big(\frac{(N n_w+n_\phi)((1-\xi )^{\frac{N}{2}}+c_\epsilon)}{\epsilon^{2}}\Big){\color{red}^\flat }$ &    $\mathcal{O}\Big(\frac{(n^2_w+n_w n_\phi) N ((1-\xi )^{\frac{N}{2}}+ c_\epsilon )}{\epsilon^{2}}\Big)$ \\   
		ANIL (Nonconvex) &$\mathcal{O}\Big(\frac{(N n_w+n_\phi)N}{\epsilon^{2}}\Big)$ &   $\mathcal{O}\Big(\frac{(n^2_w+n_wn_\phi)N^2}{\epsilon^{2}}\Big)$ \\  
		\bottomrule	
		\multicolumn{3}{l}{%
  \begin{minipage}{13cm}%
  \vspace{0.1cm}
    \footnotesize ${\color{red}^\sharp }$: with respect to each dimension of gradient. ${\color{red}^\S}$: with respect to each entry of second-order derivatives.
    \\ ${\color{red}^\aleph }$: $\kappa$ is the inner-loop stepsize used in MAML.
    ${\color{red}^\flat}:$ $\xi=\frac{\mu^2}{L^2}<1$ and  $c_\epsilon$ is a constant.
      \end{minipage}%
}\\
	\end{tabular} 
	\label{table:maml_anil}
	\vspace{-0.1cm}
\end{table*} 

For ANIL with the strongly-convex inner loop, \Cref{table:maml_anil} shows that ANIL requires fewer  gradient and second-order entry evaluations than MAML by a factor of {\small$\mathcal{O}\big(\frac{N n_w+N n_\phi}{N n_w+n_\phi}\big(1+\kappa L\big)^N\big)$} and {\small$\mathcal{O}\big(\frac{n_w+n_\phi}{n_w}\big(1+\kappa L\big)^N\big)$}, respectively. Such improvements are significant because  $n_\phi$ is often much larger than $n_w$.

For nonconvex inner loop, we set $\kappa\leq 1/N$ for MAML~\cite[Corollary 2]{ji2020multi} to be consistent with our analysis for ANIL in \Cref{th:nonconvex}. Then,~\Cref{table:maml_anil} indicates that 
ANIL requires fewer gradient and second-order entry computations than MAML by a factor of {\small$\mathcal{O}\big(\frac{Nn_w+Nn_\phi}{Nn_w+n_\phi}\big)$} and {\small$\mathcal{O}\big(\frac{n_w+n_\phi}{n_w}\big)$}.

%
%

\section{Experiments}
In this section, we validate our theory on the ANIL algorithm over two benchmarks for few-shot multiclass classification, i.e., FC100~\citep{oreshkin2018tadam} and miniImageNet~\citep{vinyals2016matching}. The experimental implementation and the model architectures are adapted from the existing repository~\citep{learn2learn2019} for ANIL. 
We consider a 5-way 5-shot task on both the FC100 and miniImageNet datasets. 
We relegate the introduction of datasets, model architectures and hyper-parameter settings to~\Cref{appen:exp} due to the space limitations.

Our experiments aim to explore how the different geometry (i.e., strong convexity and nonconvexity) of the inner loop affects the convergence performance of ANIL. 

\subsection{ANIL with Strongly-Convex Inner-Loop Loss}\label{exp:strongly-convex}

We first validate the convergence results of ANIL under the {\em strongly-convex} inner-loop loss function $L_{\mathcal{S}_i}(\cdot,\phi)$, as we establish in~\Cref{se:strong-convex}. 
Here, we let $w$ be parameters of {\em the last layer} of CNN and $\phi$ be parameters of the remaining inner layers. As in~\cite{bertinetto2018meta,lee2019meta}, the inner-loop loss function adopts $L^2$ regularization on $w$ with a hyper-parameter $\lambda>0$, and hence is {\em strongly convex}. 

  \begin{figure*}[h]
	\centering    
	\subfigure[dataset: FC100 ]{\label{fig1:a}\includegraphics[width=34.5mm]{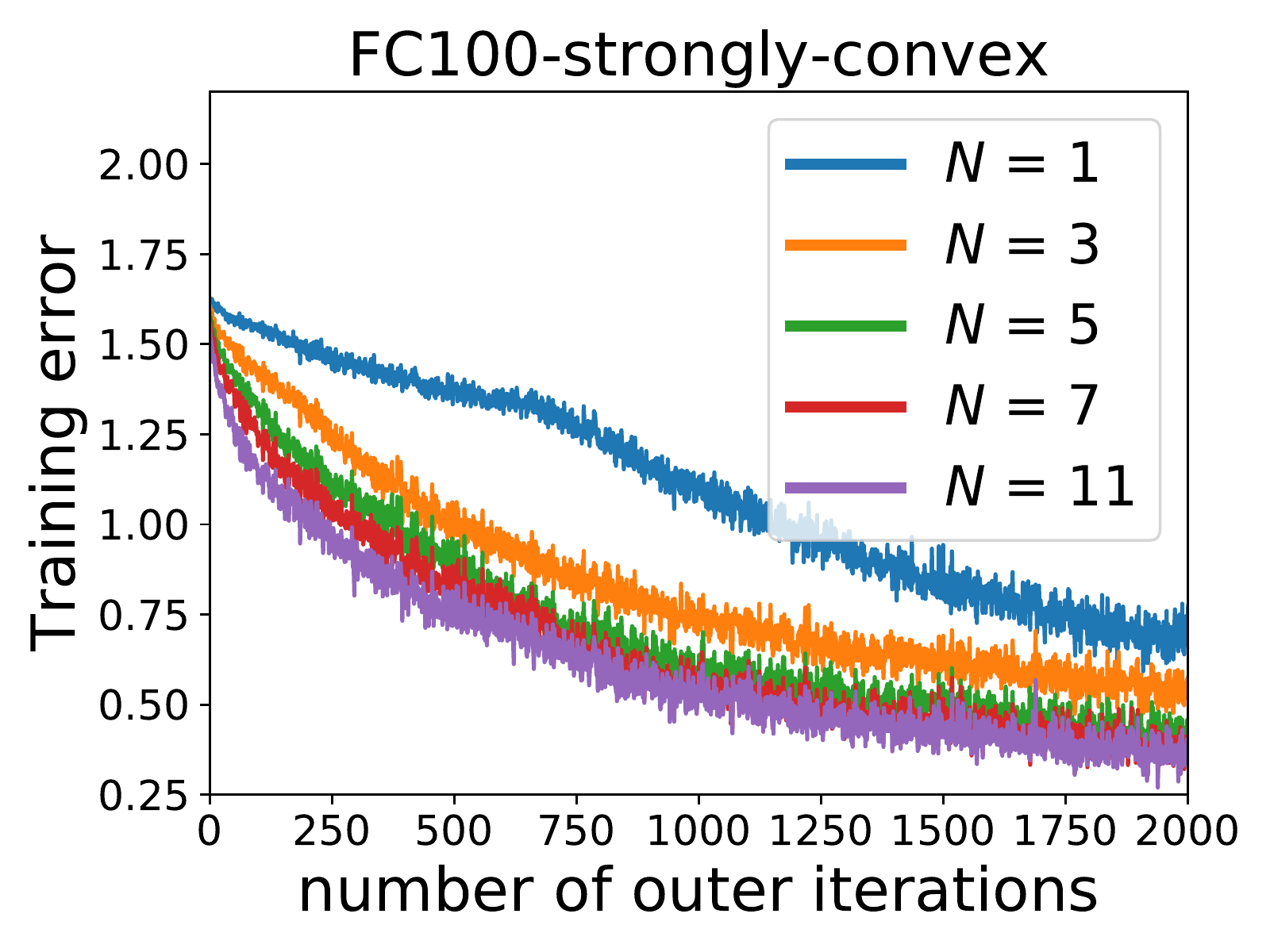}\includegraphics[width=34.5mm]{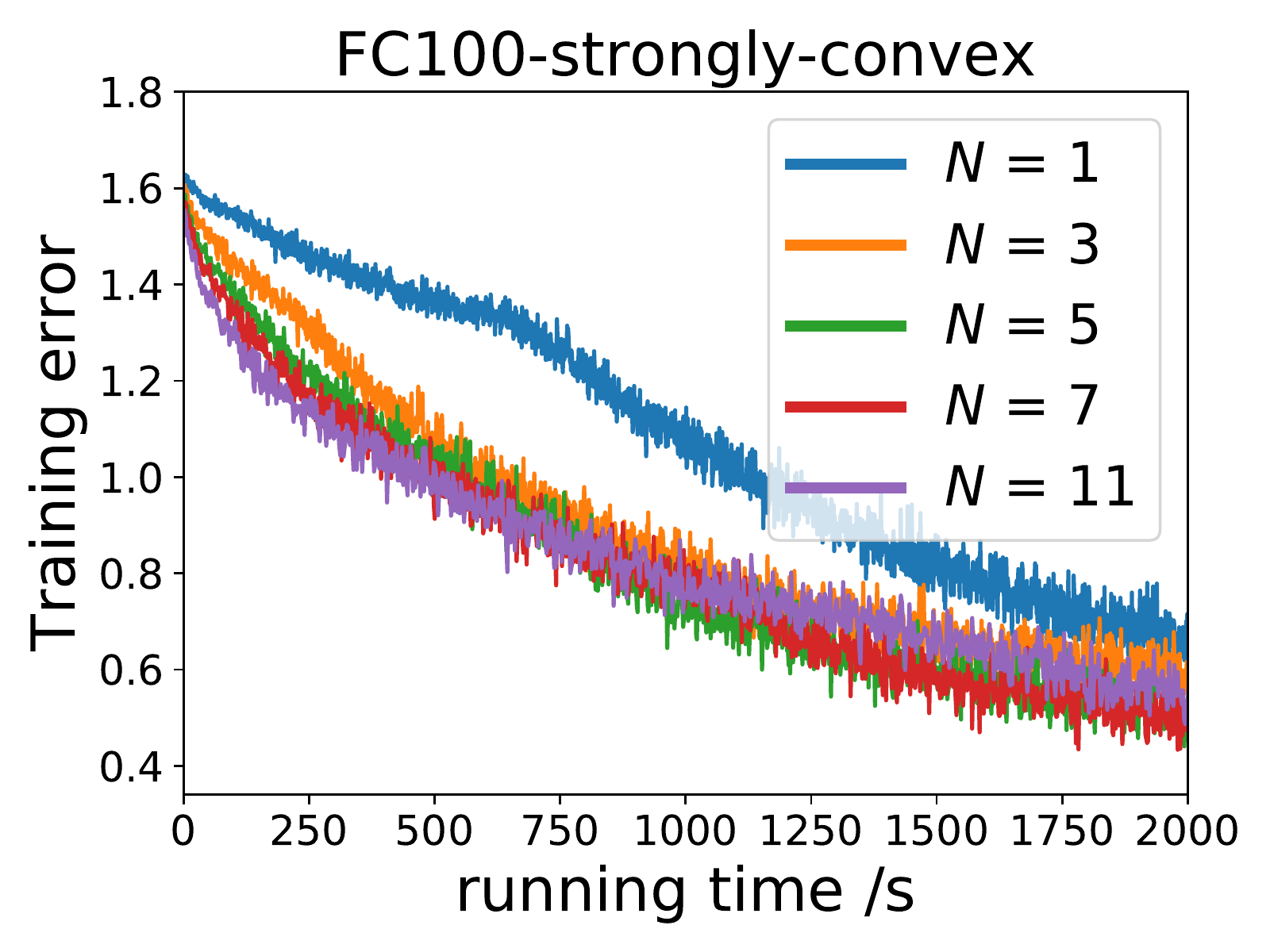}} 
	\subfigure[dataset: miniImageNet ]{\label{fig1:b}\includegraphics[width=34.5mm]{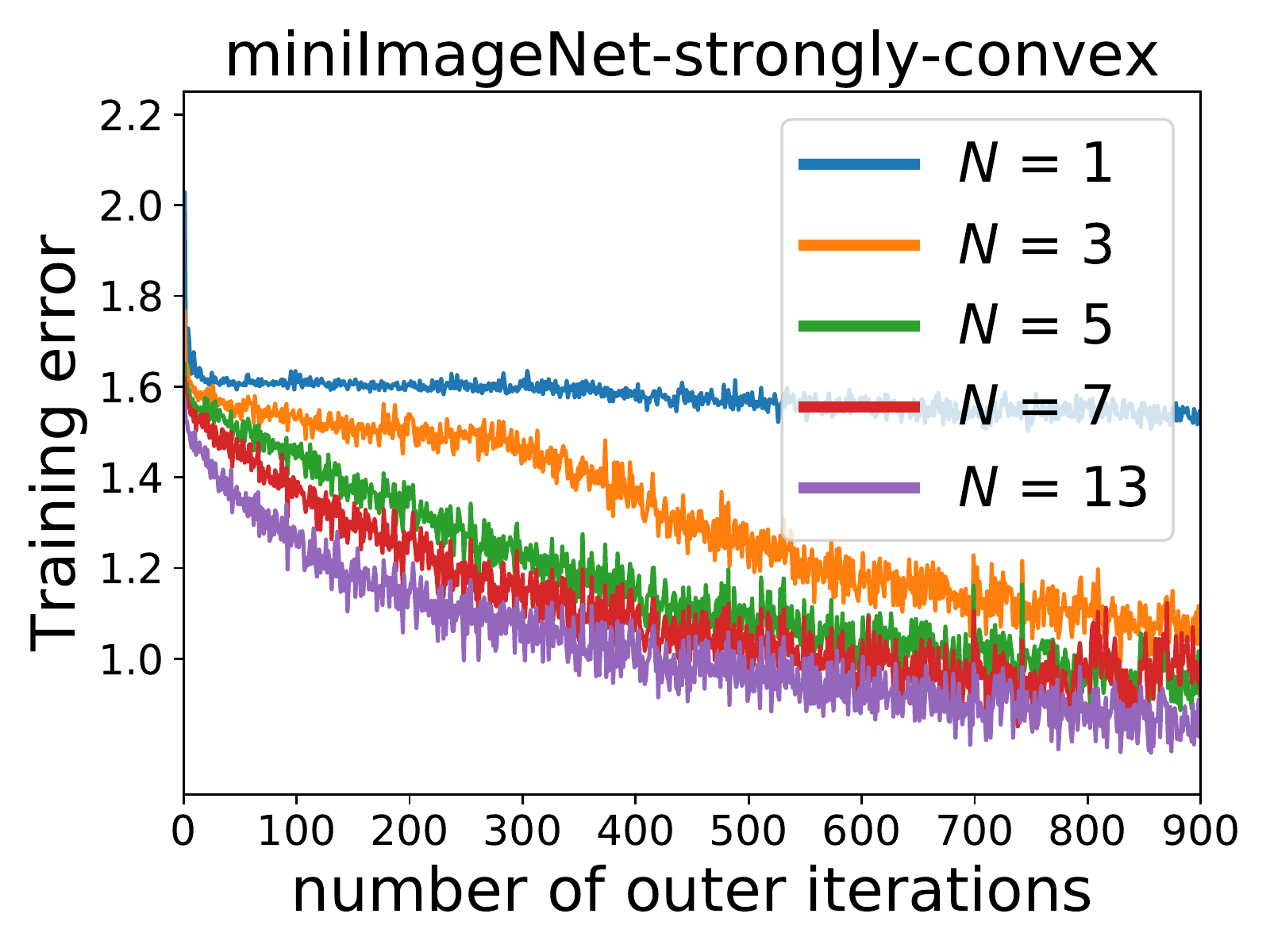}\includegraphics[width=34.5mm]{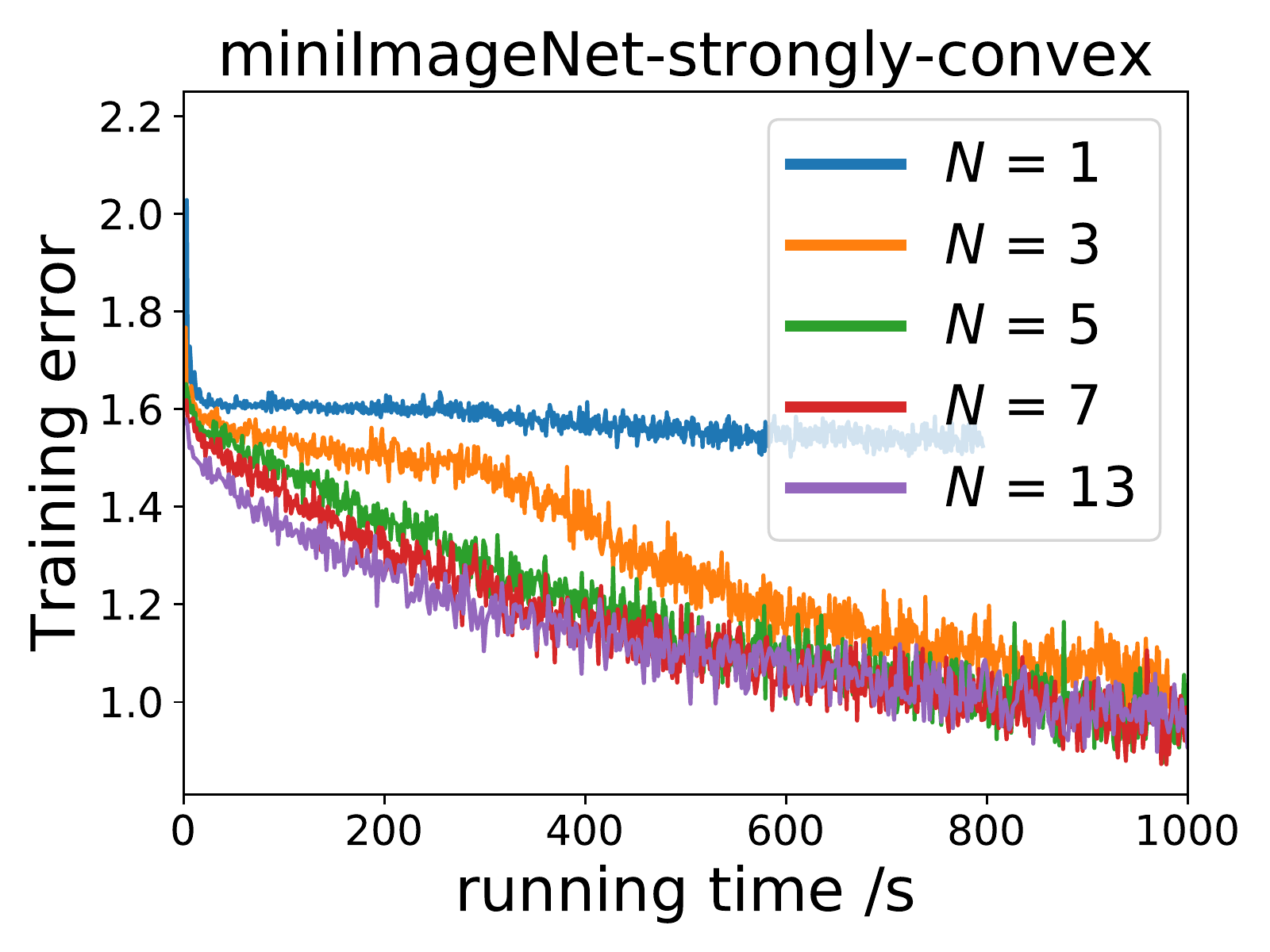}}  
	\caption{Convergence of ANIL with strongly-convex inner-loop loss function.  For each dataset, left plot: training loss v.s. number of total meta iterations; right plot: training loss v.s. running time.}\label{fig:strfc100}
\end{figure*}

For the FC100 dataset, the left plot of~\Cref{fig1:a} shows that the convergence rate in terms of the number of meta outer-loop iterations becomes faster as the inner-loop steps $N$ increases, but nearly saturates at $N=7$ (i.e., there is not much improvement for $N\geq 7$).
This is consistent with \Cref{th:strong-convex}, in which the gradient convergence bound first  
 decays exponentially with $N$, and then the bound in $\phi$ dominates and saturates to a constant. Furthermore, the right plot of~\Cref{fig1:a} shows that the running-time convergence first becomes faster  as $N$ increases up to $N\leq 7$, and then starts to slow down as $N$ further increases. 
This is also captured by~\Cref{th:strong-convex} as follows. 
The computational cost of ANIL initially decreases because the exponential reduction dominates the linear growth in the gradient and second-order derivative evaluations. But when $N$ becomes large enough, the linear growth dominates, and hence the overall computational cost of ANIL gets higher as $N$ further increases. 
Similar nature of convergence behavior is also observed over the miniImageNet dataset as shown in~\Cref{fig1:b}.  Thus, our experiment suggests that for the strongly-convex inner-loop loss, choosing a relatively large $N$ (e.g., $N=7$) 
achieves a good balance between the convergence rate (as well as the convergence error) and the computational complexity.  


\subsection{ANIL with Nonconvex Inner-Loop Loss}\label{exp:nonconvex}

We next validate the convergence results of ANIL under the {\em nonconvex} inner-loop loss function $L_{\mathcal{S}_i}(\cdot,\phi)$, as we establish in~\Cref{sec:nonconvex}. 
Here, we let $w$ be the parameters of {\em the last two layers with ReLU activation} of CNN (and hence the inner-loop loss is nonconvex with respect to $w$) and $\phi$ be the remaining parameters of the inner layers. 

  \begin{figure*}[h]
	\centering    
	\subfigure[dataset: FC100 ]{\label{fig2:a}\includegraphics[width=34.5mm]{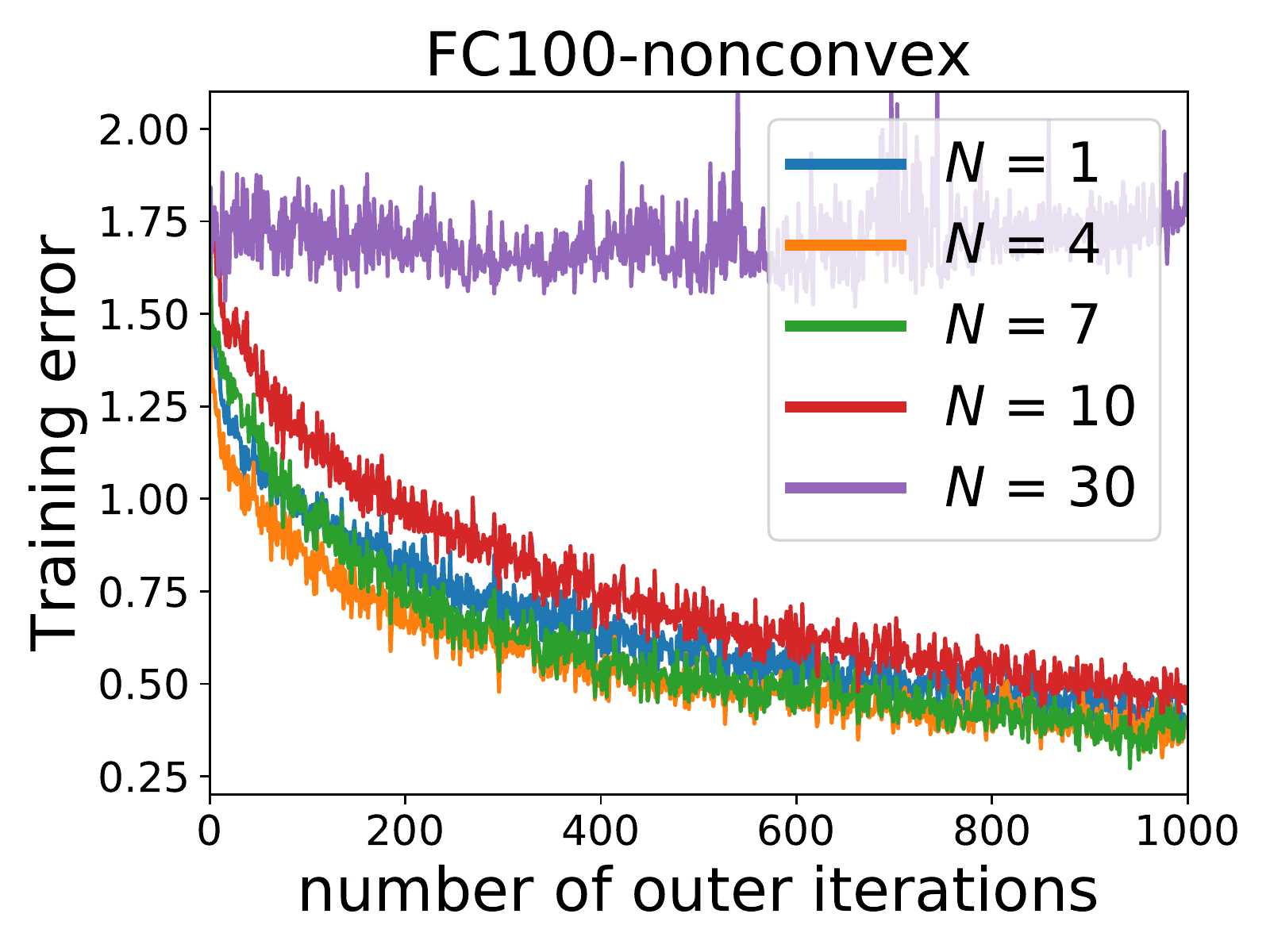}\includegraphics[width=34.5mm]{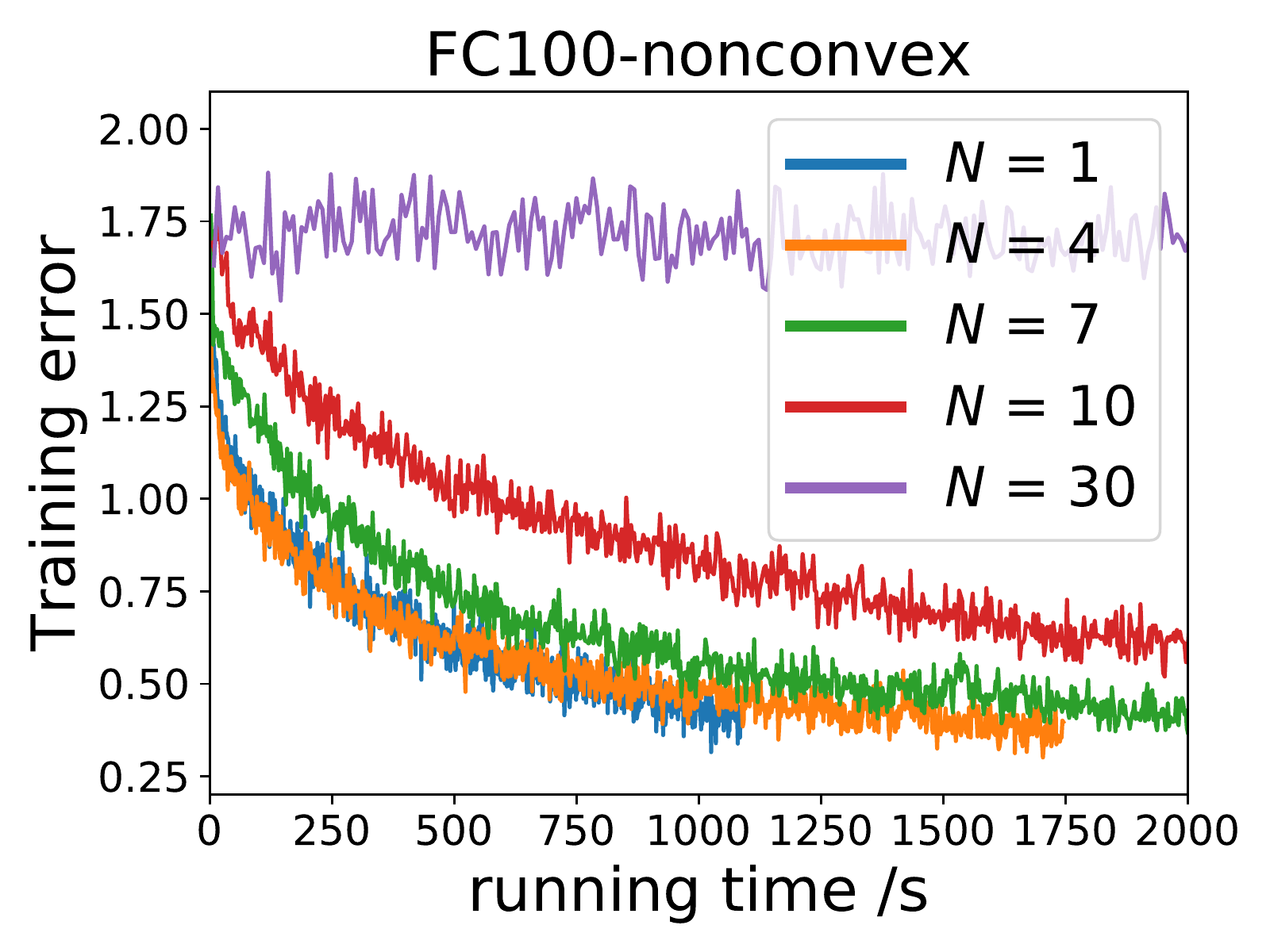}} 
	\subfigure[dataset: miniImageNet ]{\label{fig2:b}\includegraphics[width=34.5mm]{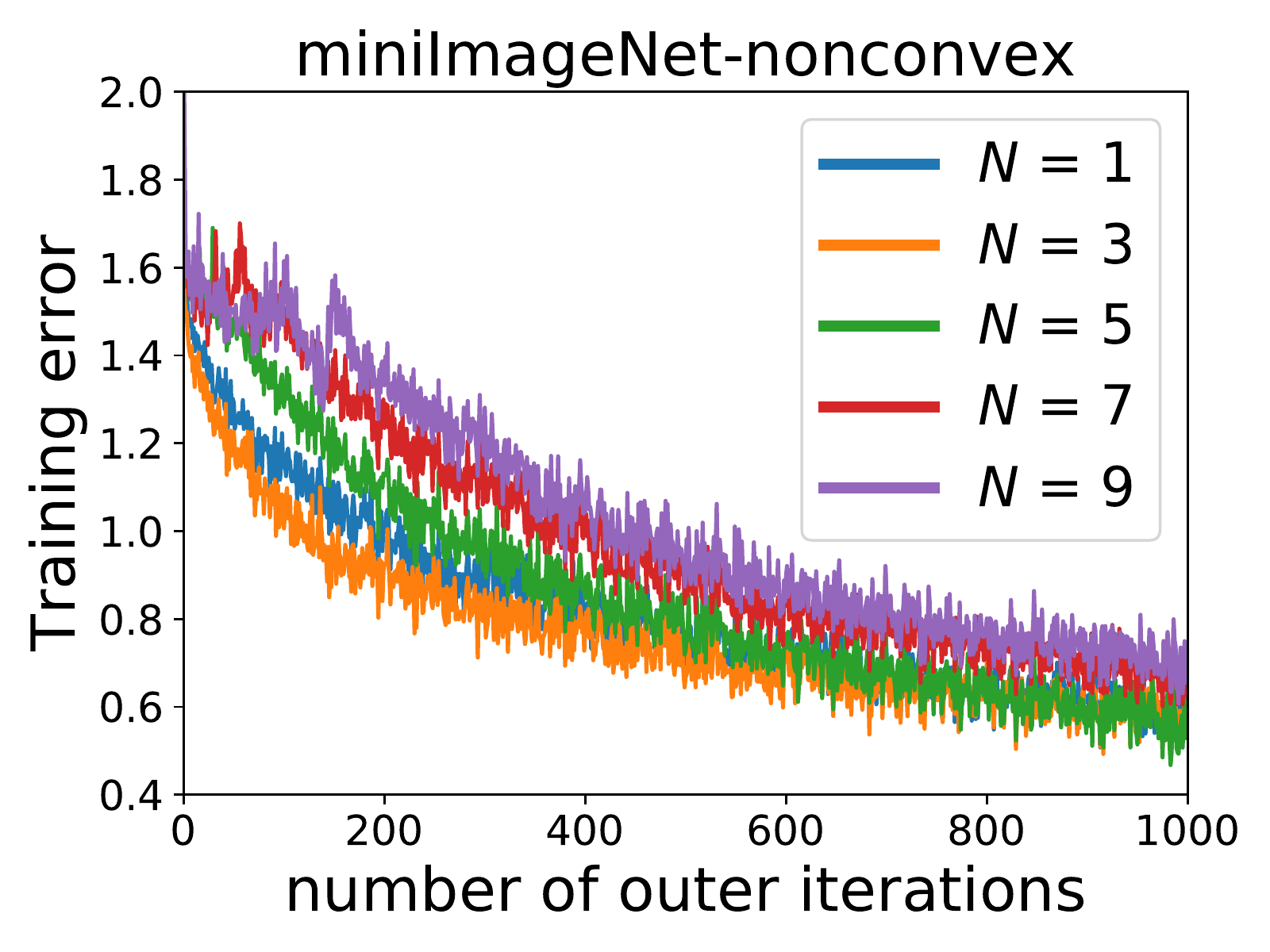}\includegraphics[width=34.5mm]{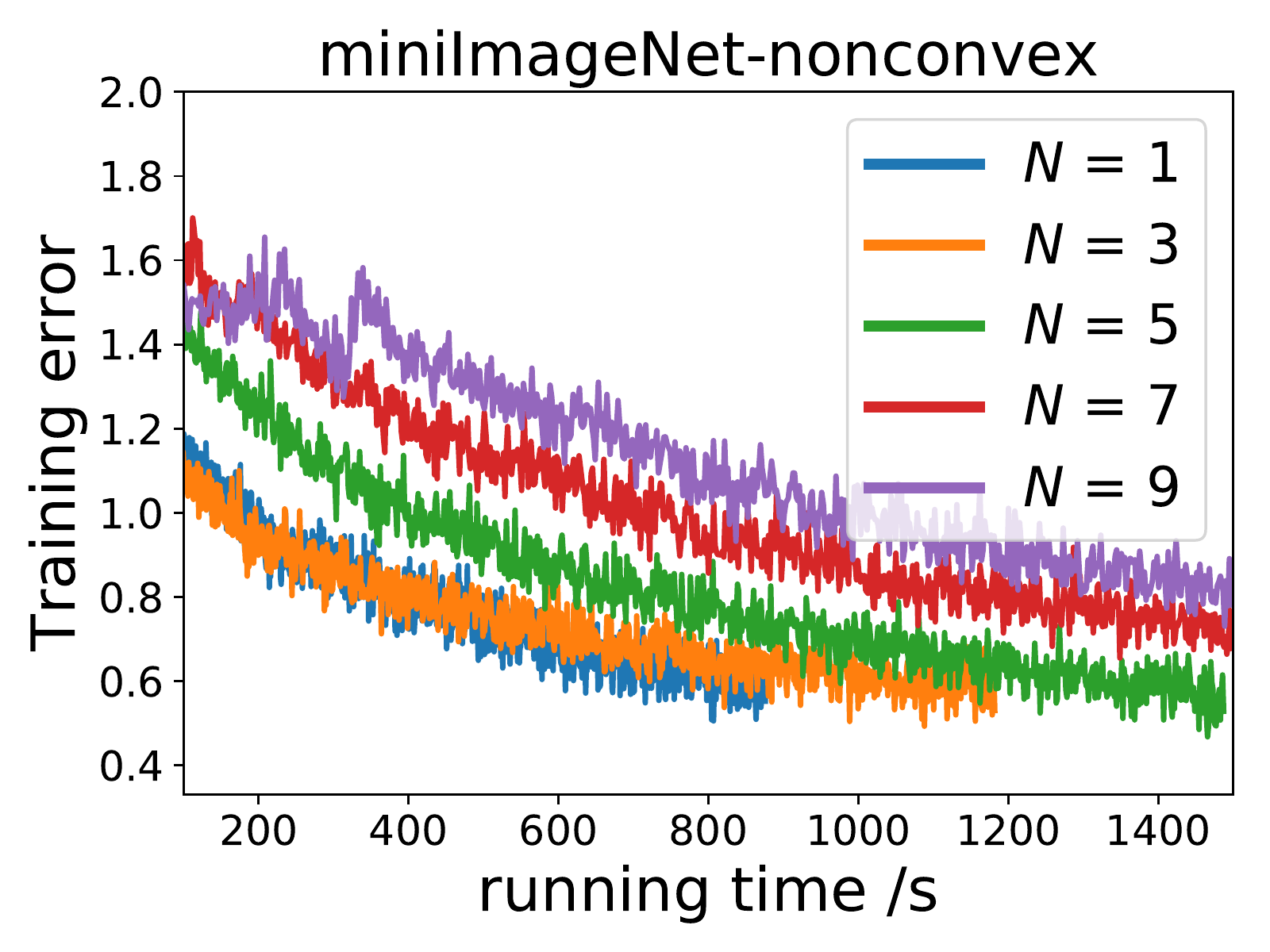}}  
	\caption{Convergence of ANIL with nonconvex inner-loop loss function.  For each dataset, left plot: training loss v.s. number of total meta iterations; right plot: training loss v.s. running time.}\label{figure:result2}
\end{figure*}

\Cref{figure:result2} provides the experimental results over the datasets FC100 and miniImageNet. 
For both datasets, the running-time convergence (right plot for each dataset) 
becomes {\em slower} as $N$ increases, where $N=1$ is fastest, and the algorithm even diverges for $N=30$ over the FC100 dataset. The plots are consist with \Cref{th:nonconvex}, in which the computational complexity increases as $N$ becomes large. Note that $N=1$ is not the fastest in the left plot for each dataset because the influence of $N$ is more prominent in terms of the running time than the number of outer-loop iterations (which is likely offset by other constant-level parameters for small $N$).
Thus, the optimization perspective here suggests that $N$ should be chosen as small as possible for computational efficiency, which in practice should be jointly considered with other aspects such as generalization for determining $N$. 


\section{Conclusion}
In this paper, we provide theoretical convergence guarantee for the ANIL algorithm under strongly-convex and nonconvex inner-loop loss functions, respectively. Our analysis reveals different performance behaviors of ANIL under the two geometries by characterizing the impact of inner-loop adaptation steps on the overall convergence rate. Our results further provide guidelines for the hyper-parameter selections for ANIL under different inner-loop loss geometries. 

\section*{Broader Impact}
Meta-learning has been successfully used in a wide range of applications including reinforcement learning, robotics, federated learning, imitation learning, etc, which will be highly influential to technologize our life. This work focuses on understanding the computational efficiency of the optimization-based meta learning algorithms, particularly MAML and ANIL type algorithms. We characterize the convergence guarantee on these algorithms. Furthermore, our theory provides useful guidelines on the selections of hyperparameters for these algorithms, in order for them to be efficiently implemented in large-scale applications. We also anticipate the theory that we develop will be useful in other academic fields in addition to machine learning, including optimization theory, signal processing, and statistics.
\begin{ack}
The work of K. Ji and Y. Liang is supported in part by the U.S. National Science Foundation under the grants CCF-1900145 and CCF-1761506. JDL acknowledges support of NSF CCF 2002272. 
\end{ack}

\bibliographystyle{ims}
\bibliography{ref}

\begin{thebibliography}{34}
\expandafter\ifx\csname natexlab\endcsname\relax\def\natexlab#1{#1}\fi
\expandafter\ifx\csname url\endcsname\relax
  \def\url#1{\texttt{#1}}\fi
\expandafter\ifx\csname urlprefix\endcsname\relax\def\urlprefix{URL }\fi

\bibitem[{Arnold et~al.(2019)Arnold, Mahajan, Datta and
  Bunner}]{learn2learn2019}
\textsc{Arnold, S.~M.}, \textsc{Mahajan, P.}, \textsc{Datta, D.} and
  \textsc{Bunner, I.} (2019).
\newblock \textit{learn2learn}.
\newblock \url{https://github.com/learnables/learn2learn}.

\bibitem[{Bengio et~al.(1991)Bengio, Bengio and Cloutier}]{bengio1991learning}
\textsc{Bengio, Y.}, \textsc{Bengio, S.} and \textsc{Cloutier, J.} (1991).
\newblock Learning a synaptic learning rule.
\newblock In \textit{IEEE International Joint Conference on Neural Networks
  (IJCNN)}.

\bibitem[{Bertinetto et~al.(2018)Bertinetto, Henriques, Torr and
  Vedaldi}]{bertinetto2018meta}
\textsc{Bertinetto, L.}, \textsc{Henriques, J.~F.}, \textsc{Torr, P.} and
  \textsc{Vedaldi, A.} (2018).
\newblock Meta-learning with differentiable closed-form solvers.
\newblock In \textit{International Conference on Learning Representations
  (ICLR)}.

\bibitem[{Fallah et~al.(2019)Fallah, Mokhtari and
  Ozdaglar}]{fallah2019convergence}
\textsc{Fallah, A.}, \textsc{Mokhtari, A.} and \textsc{Ozdaglar, A.} (2019).
\newblock On the convergence theory of gradient-based model-agnostic
  meta-learning algorithms.
\newblock \textit{arXiv preprint arXiv:1908.10400} .

\bibitem[{Finn et~al.(2017{\natexlab{a}})Finn, Abbeel and
  Levine}]{finn2017model}
\textsc{Finn, C.}, \textsc{Abbeel, P.} and \textsc{Levine, S.}
  (2017{\natexlab{a}}).
\newblock Model-agnostic meta-learning for fast adaptation of deep networks.
\newblock In \textit{Proc. International Conference on Machine Learning
  (ICML)}.

\bibitem[{Finn and Levine(2017)}]{finn2017meta}
\textsc{Finn, C.} and \textsc{Levine, S.} (2017).
\newblock Meta-learning and universality: Deep representations and gradient
  descent can approximate any learning algorithm.
\newblock \textit{International Conference on Learning Representations (ICLR)}
  .

\bibitem[{Finn et~al.(2019)Finn, Rajeswaran, Kakade and
  Levine}]{finn2019online}
\textsc{Finn, C.}, \textsc{Rajeswaran, A.}, \textsc{Kakade, S.} and
  \textsc{Levine, S.} (2019).
\newblock Online meta-learning.
\newblock In \textit{International Conference on Machine Learning (ICML)}.

\bibitem[{Finn et~al.(2018)Finn, Xu and Levine}]{finn2018probabilistic}
\textsc{Finn, C.}, \textsc{Xu, K.} and \textsc{Levine, S.} (2018).
\newblock Probabilistic model-agnostic meta-learning.
\newblock In \textit{Advances in Neural Information Processing Systems
  (NeurIPS)}.

\bibitem[{Finn et~al.(2017{\natexlab{b}})Finn, Yu, Zhang, Abbeel and
  Levine}]{finn2017one}
\textsc{Finn, C.}, \textsc{Yu, T.}, \textsc{Zhang, T.}, \textsc{Abbeel, P.} and
  \textsc{Levine, S.} (2017{\natexlab{b}}).
\newblock One-shot visual imitation learning via meta-learning.
\newblock In \textit{Conference on Robot Learning (CoRL)}.

\bibitem[{Ghadimi and Lan(2013)}]{ghadimi2013stochastic}
\textsc{Ghadimi, S.} and \textsc{Lan, G.} (2013).
\newblock Stochastic first-and zeroth-order methods for nonconvex stochastic
  programming.
\newblock \textit{SIAM Journal on Optimization} \textbf{23} 2341--2368.

\bibitem[{Ghadimi and Wang(2018)}]{ghadimi2018approximation}
\textsc{Ghadimi, S.} and \textsc{Wang, M.} (2018).
\newblock Approximation methods for bilevel programming.
\newblock \textit{arXiv preprint arXiv:1802.02246} .

\bibitem[{Jerfel et~al.(2018)Jerfel, Grant, Griffiths and
  Heller}]{jerfel2018online}
\textsc{Jerfel, G.}, \textsc{Grant, E.}, \textsc{Griffiths, T.~L.} and
  \textsc{Heller, K.} (2018).
\newblock Online gradient-based mixtures for transfer modulation in
  meta-learning.
\newblock \textit{arXiv preprint arXiv:1812.06080} .

\bibitem[{Ji et~al.(2019{\natexlab{a}})Ji, Wang, Weng, Zhou, Zhang and
  Liang}]{ji2019history}
\textsc{Ji, K.}, \textsc{Wang, Z.}, \textsc{Weng, B.}, \textsc{Zhou, Y.},
  \textsc{Zhang, W.} and \textsc{Liang, Y.} (2019{\natexlab{a}}).
\newblock History-gradient aided batch size adaptation for variance reduced
  algorithms.
\newblock \textit{arXiv preprint arXiv:1910.09670} .

\bibitem[{Ji et~al.(2019{\natexlab{b}})Ji, Wang, Zhou and
  Liang}]{ji2019improved}
\textsc{Ji, K.}, \textsc{Wang, Z.}, \textsc{Zhou, Y.} and \textsc{Liang, Y.}
  (2019{\natexlab{b}}).
\newblock Improved zeroth-order variance reduced algorithms and analysis for
  nonconvex optimization.
\newblock In \textit{International Conference on Machine Learning (ICML)}.

\bibitem[{Ji et~al.(2020)Ji, Yang and Liang}]{ji2020multi}
\textsc{Ji, K.}, \textsc{Yang, J.} and \textsc{Liang, Y.} (2020).
\newblock Multi-step model-agnostic meta-learning: Convergence and improved
  algorithms.
\newblock \textit{arXiv preprint arXiv:2002.07836} .

\bibitem[{Kingma and Ba(2014)}]{kingma2014adam}
\textsc{Kingma, D.~P.} and \textsc{Ba, J.} (2014).
\newblock Adam: A method for stochastic optimization.
\newblock \textit{International Conference on Learning Representations (ICLR)}
  .

\bibitem[{Koch et~al.(2015)Koch, Zemel and Salakhutdinov}]{koch2015siamese}
\textsc{Koch, G.}, \textsc{Zemel, R.} and \textsc{Salakhutdinov, R.} (2015).
\newblock Siamese neural networks for one-shot image recognition.
\newblock In \textit{ICML Deep Learning Workshop}, vol.~2.

\bibitem[{Krizhevsky and Hinton(2009)}]{krizhevsky2009learning}
\textsc{Krizhevsky, A.} and \textsc{Hinton, G.} (2009).
\newblock Learning multiple layers of features from tiny images .

\bibitem[{Lee et~al.(2019)Lee, Maji, Ravichandran and Soatto}]{lee2019meta}
\textsc{Lee, K.}, \textsc{Maji, S.}, \textsc{Ravichandran, A.} and
  \textsc{Soatto, S.} (2019).
\newblock Meta-learning with differentiable convex optimization.
\newblock In \textit{IEEE Conference on Computer Vision and Pattern Recognition
  (CVPR)}.

\bibitem[{Lin et~al.(2020)Lin, Jin, Jordan et~al.}]{lin2020near}
\textsc{Lin, T.}, \textsc{Jin, C.}, \textsc{Jordan, M.} \textsc{et~al.} (2020).
\newblock Near-optimal algorithms for minimax optimization.
\newblock \textit{arXiv preprint arXiv:2002.02417} .

\bibitem[{Mi et~al.(2019)Mi, Huang, Zhang and Faltings}]{mi2019meta}
\textsc{Mi, F.}, \textsc{Huang, M.}, \textsc{Zhang, J.} and \textsc{Faltings,
  B.} (2019).
\newblock Meta-learning for low-resource natural language generation in
  task-oriented dialogue systems.
\newblock In \textit{International Joint Conference on Artificial Intelligence
  (IJCAI)}.

\bibitem[{Munkhdalai and Yu(2017)}]{munkhdalai2017meta}
\textsc{Munkhdalai, T.} and \textsc{Yu, H.} (2017).
\newblock Meta networks.
\newblock In \textit{International Conference on Machine Learning (ICML)}.

\bibitem[{Nichol and Schulman(2018)}]{nichol2018reptile}
\textsc{Nichol, A.} and \textsc{Schulman, J.} (2018).
\newblock Reptile: a scalable metalearning algorithm.
\newblock \textit{arXiv preprint arXiv:1803.02999} .

\bibitem[{Oreshkin et~al.(2018)Oreshkin, L{\'o}pez and
  Lacoste}]{oreshkin2018tadam}
\textsc{Oreshkin, B.}, \textsc{L{\'o}pez, P.~R.} and \textsc{Lacoste, A.}
  (2018).
\newblock Tadam: Task dependent adaptive metric for improved few-shot learning.
\newblock In \textit{Advances in Neural Information Processing Systems
  (NeurIPS)}.

\bibitem[{Raghu et~al.(2019)Raghu, Raghu, Bengio and Vinyals}]{raghu2019rapid}
\textsc{Raghu, A.}, \textsc{Raghu, M.}, \textsc{Bengio, S.} and
  \textsc{Vinyals, O.} (2019).
\newblock Rapid learning or feature reuse? towards understanding the
  effectiveness of {MAML}.
\newblock \textit{International Conference on Learning Representations (ICLR)}
  .

\bibitem[{Rajeswaran et~al.(2019)Rajeswaran, Finn, Kakade and
  Levine}]{rajeswaran2019meta}
\textsc{Rajeswaran, A.}, \textsc{Finn, C.}, \textsc{Kakade, S.~M.} and
  \textsc{Levine, S.} (2019).
\newblock Meta-learning with implicit gradients.
\newblock In \textit{Advances in Neural Information Processing Systems
  (NeurIPS)}.

\bibitem[{Ravi and Larochelle(2016)}]{ravi2016optimization}
\textsc{Ravi, S.} and \textsc{Larochelle, H.} (2016).
\newblock Optimization as a model for few-shot learning.
\newblock In \textit{International Conference on Learning Representations
  (ICLR)}.

\bibitem[{Russakovsky et~al.(2015)Russakovsky, Deng, Su, Krause, Satheesh, Ma,
  Huang, Karpathy, Khosla, Bernstein, Berg and
  Fei-Fei}]{russakovsky2015imagenet}
\textsc{Russakovsky, O.}, \textsc{Deng, J.}, \textsc{Su, H.}, \textsc{Krause,
  J.}, \textsc{Satheesh, S.}, \textsc{Ma, S.}, \textsc{Huang, Z.},
  \textsc{Karpathy, A.}, \textsc{Khosla, A.}, \textsc{Bernstein, M.},
  \textsc{Berg, A.~C.} and \textsc{Fei-Fei, L.} (2015).
\newblock Imagenet large scale visual recognition challenge.
\newblock \textit{International Journal of Computer Vision} \textbf{3}
  211--252.

\bibitem[{Snell et~al.(2017)Snell, Swersky and Zemel}]{snell2017prototypical}
\textsc{Snell, J.}, \textsc{Swersky, K.} and \textsc{Zemel, R.} (2017).
\newblock Prototypical networks for few-shot learning.
\newblock In \textit{Advances in Neural Information Processing Systems (NIPS)}.

\bibitem[{Thrun and Pratt(2012)}]{thrun2012learning}
\textsc{Thrun, S.} and \textsc{Pratt, L.} (2012).
\newblock \textit{Learning to learn}.
\newblock Springer Science \& Business Media.

\bibitem[{Vinyals et~al.(2016)Vinyals, Blundell, Lillicrap and
  Wierstra}]{vinyals2016matching}
\textsc{Vinyals, O.}, \textsc{Blundell, C.}, \textsc{Lillicrap, T.} and
  \textsc{Wierstra, D.} (2016).
\newblock Matching networks for one shot learning.
\newblock In \textit{Advances in Neural Information Processing Systems (NIPS)}.

\bibitem[{Wang et~al.(2018)Wang, Ji, Zhou, Liang and
  Tarokh}]{wang2018spiderboost}
\textsc{Wang, Z.}, \textsc{Ji, K.}, \textsc{Zhou, Y.}, \textsc{Liang, Y.} and
  \textsc{Tarokh, V.} (2018).
\newblock {SpiderBoost}: A class of faster variance-reduced algorithms for
  nonconvex optimization.
\newblock \textit{arXiv preprint arXiv:1810.10690} .

\bibitem[{Zhou et~al.(2018)Zhou, Wu and Li}]{zhou2018deep}
\textsc{Zhou, F.}, \textsc{Wu, B.} and \textsc{Li, Z.} (2018).
\newblock Deep meta-learning: Learning to learn in the concept space.
\newblock \textit{arXiv preprint arXiv:1802.03596} .

\bibitem[{Zhou et~al.(2019)Zhou, Yuan, Xu, Yan and Feng}]{zhou2019efficient}
\textsc{Zhou, P.}, \textsc{Yuan, X.}, \textsc{Xu, H.}, \textsc{Yan, S.} and
  \textsc{Feng, J.} (2019).
\newblock Efficient meta learning via minibatch proximal update.
\newblock In \textit{Advances in Neural Information Processing Systems
  (NeurIPS)}.

\end{thebibliography}
\newpage
\appendix 
{\Large{\bf Supplementary Materials}}
\section{Further Specification of Experiments}\label{appen:exp}
Following~\cite{learn2learn2019}, we consider a 5-way 5-shot task on both the FC100 and miniImageNet datasets, where we evaluate the model's ability to discriminate $5$ unseen classes, given only $5$ labelled samples per class. We adopt Adam~\cite{kingma2014adam} as the optimizer for the meta outer-loop update, and adopt the cross-entropy loss to measure the error between the predicted and true labels. 
\subsection{Introduction of FC100 and miniImageNet datasets}

{\bf FC100 dataset.} The FC100 dataset~\citep{oreshkin2018tadam} is generated from CIFAR-100~\citep{krizhevsky2009learning}, and consists of $100$ classes with each class containing $600$ images of size $32\time 32$. Following recent work~\citep{oreshkin2018tadam,lee2019meta}, we split these $100$ classes into $60$ classes for meta-training, $20$ classes for meta-validation, and $20$ classes for meta-testing.  

{\bf miniImageNet dataset.} The miniImageNet dataset~\citep{vinyals2016matching} consists of $100$ classes randomly chosen from ImageNet~\citep{russakovsky2015imagenet}, where each class contains $600$ images of size $84\times 84$. Following the repository~\citep{learn2learn2019}, we partition these classes into $64$ classes for meta-training, $16$ classes for meta-validation, and $20$ classes for meta-testing.

\subsection{Model Architectures and Hyper-Parameter Setting}\label{model_structure}
We adopt the following four model architectures depending on the dataset and the geometry of the inner-loop loss. The hyper-parameter configuration for each architecture is also provided as follows. 

{\bf Case 1: FC100 dataset, strongly-convex inner-loop loss.} Following~\cite{learn2learn2019}, we use a $4$-layer CNN of four convolutional blocks, where each block sequentially consists of  a $3\times 3$ convolution with a padding of $1$ and a stride of $2$, batch normalization, ReLU activation, and $2\times 2$
max pooling. Each convolutional layer has $64$ filters. This model is trained with an inner-loop stepsize of $0.005$, an outer-loop (meta) stepsize of $0.001$, and a mini-batch size of $B=32$. We set the regularization parameter $\lambda$ of the $L^2$ regularizer  to be $\lambda =5$.

{\bf Case 2: FC100 dataset, nonconvex inner-loop loss.} We adopt a $5$-layer CNN 
 with the first four convolutional layers the same as in {\bf Case 1}, followed by ReLU activation, and a full-connected layer with size of $256\times \text{ways}$.  This model is trained with an inner-loop stepsize of $0.04$, an outer-loop (meta) stepsize of $0.003$, and a mini-batch size of $B=32$. 
 
 {\bf Case 3: miniImageNet  dataset, strongly-convex inner-loop loss.} Following~\cite{raghu2019rapid}, we use a $4$-layer CNN of four convolutional blocks, where each block sequentially consists of  a $3\times 3$ convolution with $32$ filters, batch normalization, ReLU activation, and $2\times 2$
max pooling. We choose an inner-loop stepsize of $0.002$, an outer-loop (meta) stepsize of $0.002$, and a mini-batch size of $B=32$, and set the regularization parameter $\lambda$ of the $L^2$ regularizer  to be $\lambda =0.1$.
 
 {\bf Case 4: miniImageNet dataset, nonconvex inner-loop loss.} We adopt a $5$-layer CNN 
 with the first four convolutional layers the same as in {\bf Case 3}, followed by ReLU activation, and a full-connected layer with size of $128\times \text{ways}$. We choose an inner-loop stepsize of $0.02$, an outer-loop (meta) stepsize of $0.003$, and a mini-batch size of $B=32$.

\subsection{Experiments with SGD Optimizer}
The experiments in \Cref{exp:strongly-convex} and \Cref{exp:nonconvex} adopt the Adam optimizer. 
In this subsection, we conduct experiments using mini-batch stochastic gradient descent (SGD) on FC100 dataset. For both the strongly-convex and nonconvex cases, we choose  an inner-loop stepsize of $0.05$, an outer-loop (meta) stepsize of $0.05$, and a mini-batch size of $B=32$. 
The results are given in \Cref{fig:sgd}. It can be seen that 
the nature of the results remains the same as those done with the Adam optimizer.  
  \begin{figure*}[h]
	\centering    
	\includegraphics[width=52mm]{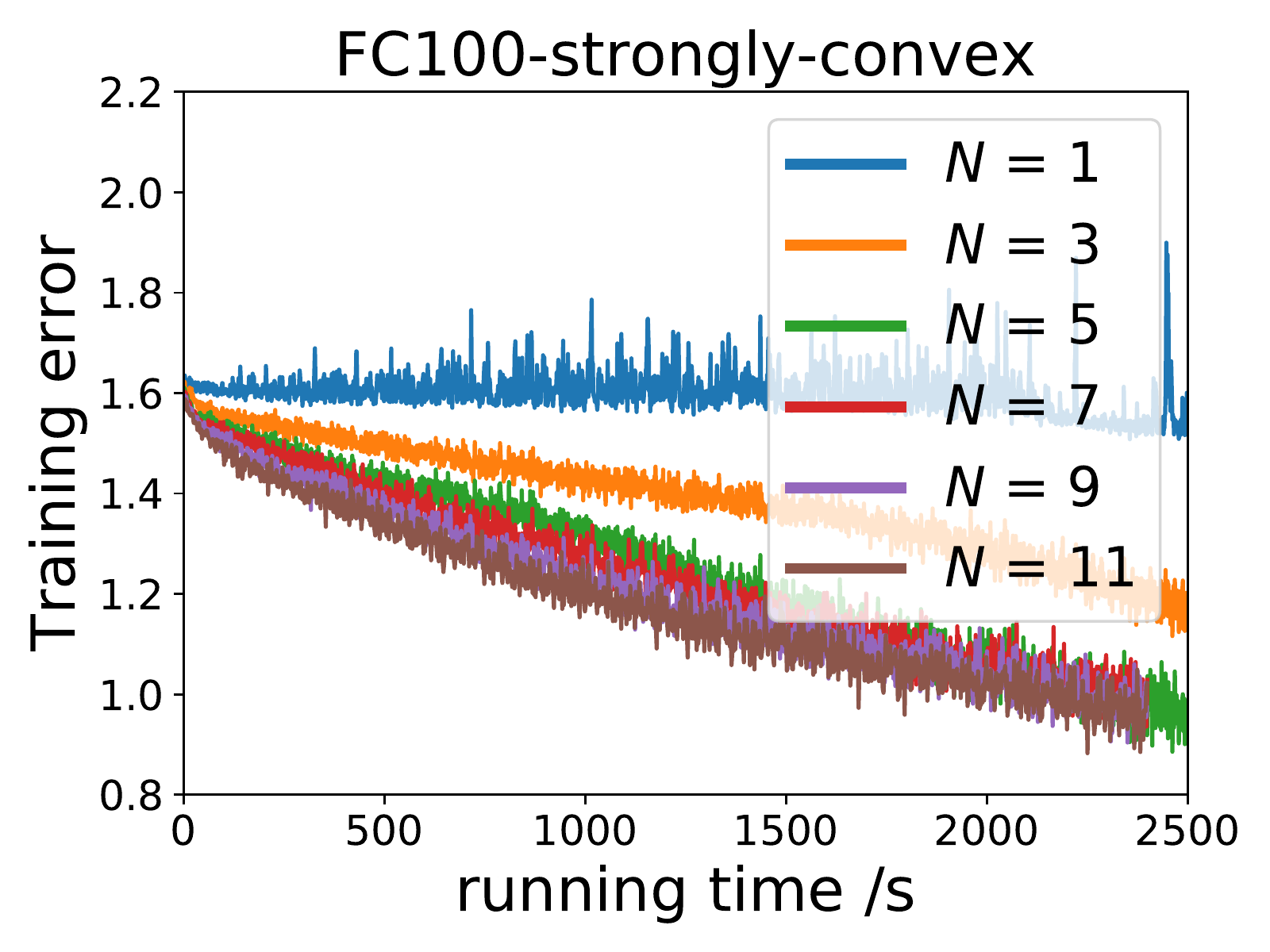}\includegraphics[width=52mm]{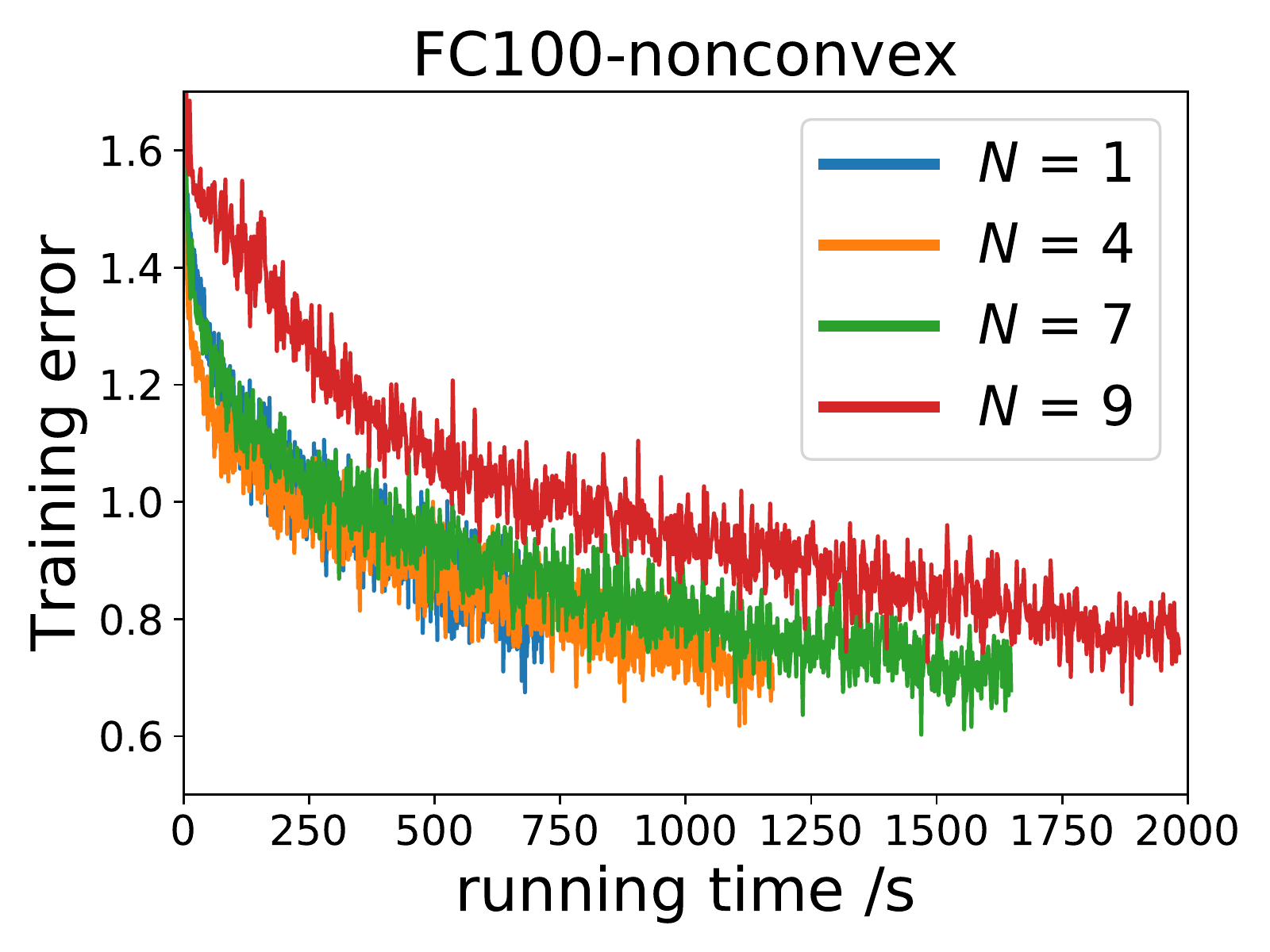}
	\caption{Convergence of ANIL with mini-batch SGD over FC100 dataset. Left plot: strongly-convex inner-loop loss; right plot: nonconvex inner-loop loss.}\label{fig:sgd}
\end{figure*}

\subsection{Experiments on Comparison of ANIL and MAML} 

In \Cref{figure:compps}, we compare the computational efficiency between ANIL and MAML. For the miniImageNet dataset, we choose the inner-loop stepsize as $0.1$, the outer-loop (meta) stepsize as $0.002$, the  mini-batch size as $32$, and the number of inner-loop steps as $5$ for ANIL. For MAML, we choose the inner-loop stepsize as $0.5$, the outer-loop stepsize as $0.003$, the  mini-batch size as $32$, and the number of inner-loop steps as $3$. 
For the FC100 dataset, we choose  the inner-loop stepsize as $0.1$, the outer-loop (meta) stepsize as $0.001$, the mini-batch size as $32$ for ANIL. 
For MAML, we choose the inner-loop stepsize as $0.5$, the outer-loop stepsize as $0.001$, and the  mini-batch size as $32$. We choose the number of inner-loop steps as $10$ for ANIL and $3$ for MAML.
It  can be seen that ANIL converges faster than MAML, as well supported by our theoretical results. 

  \begin{figure*}[h]
	\centering    
	\subfigure[dataset: FC100 ]{\label{figcc:a}\includegraphics[width=34.5mm]{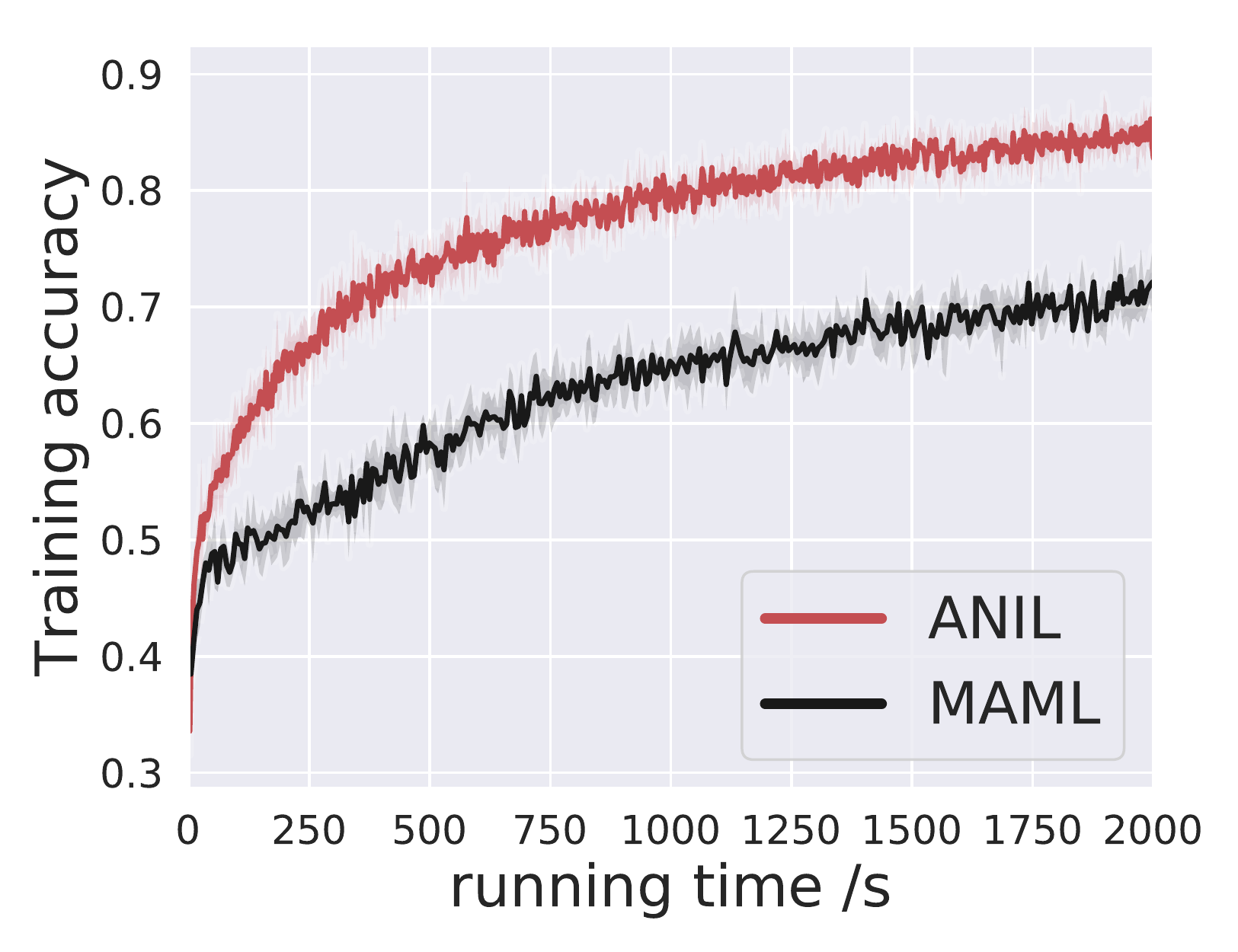}\includegraphics[width=34.5mm]{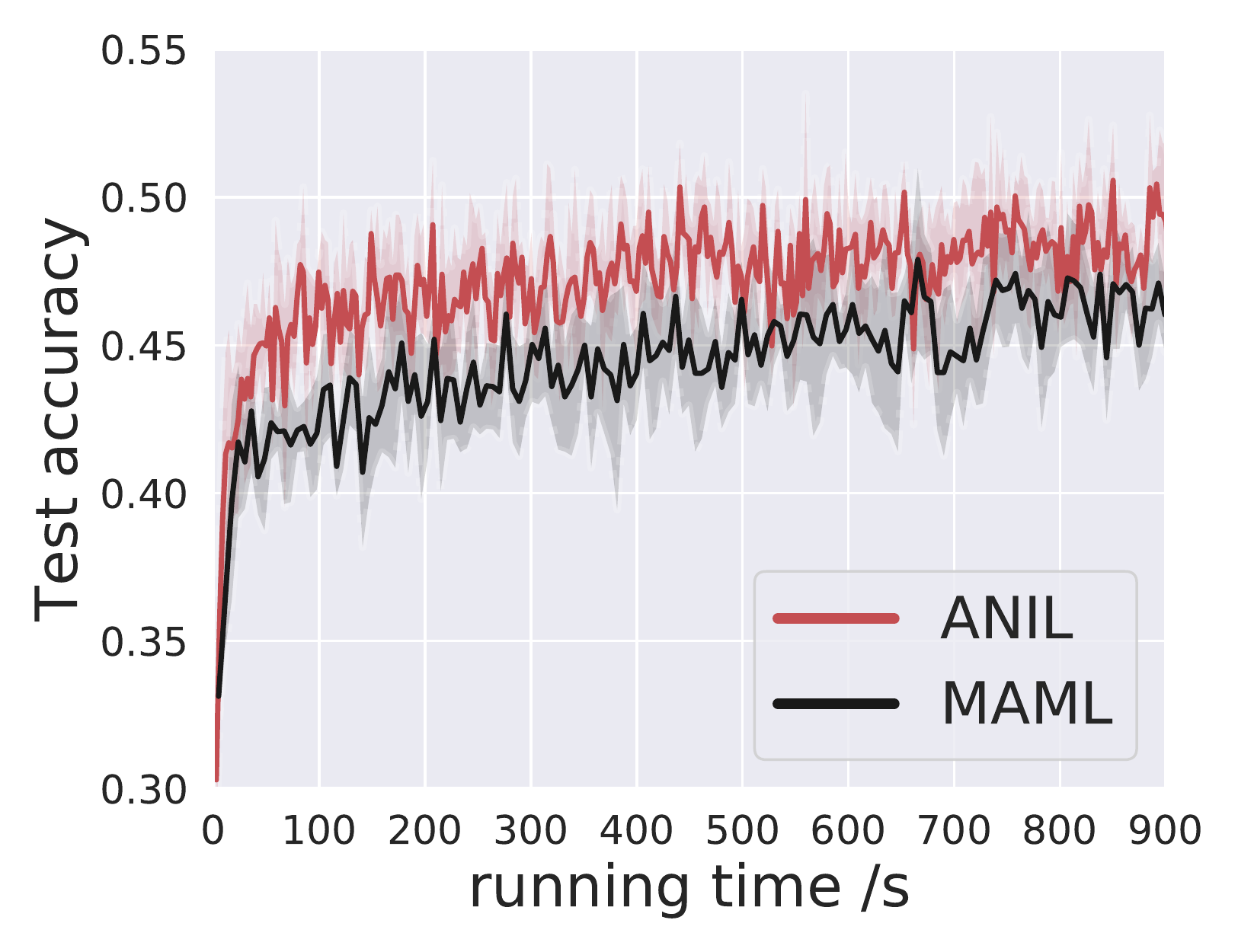}} 
	\subfigure[dataset: miniImageNet ]{\label{figcc:b}\includegraphics[width=34.5mm]{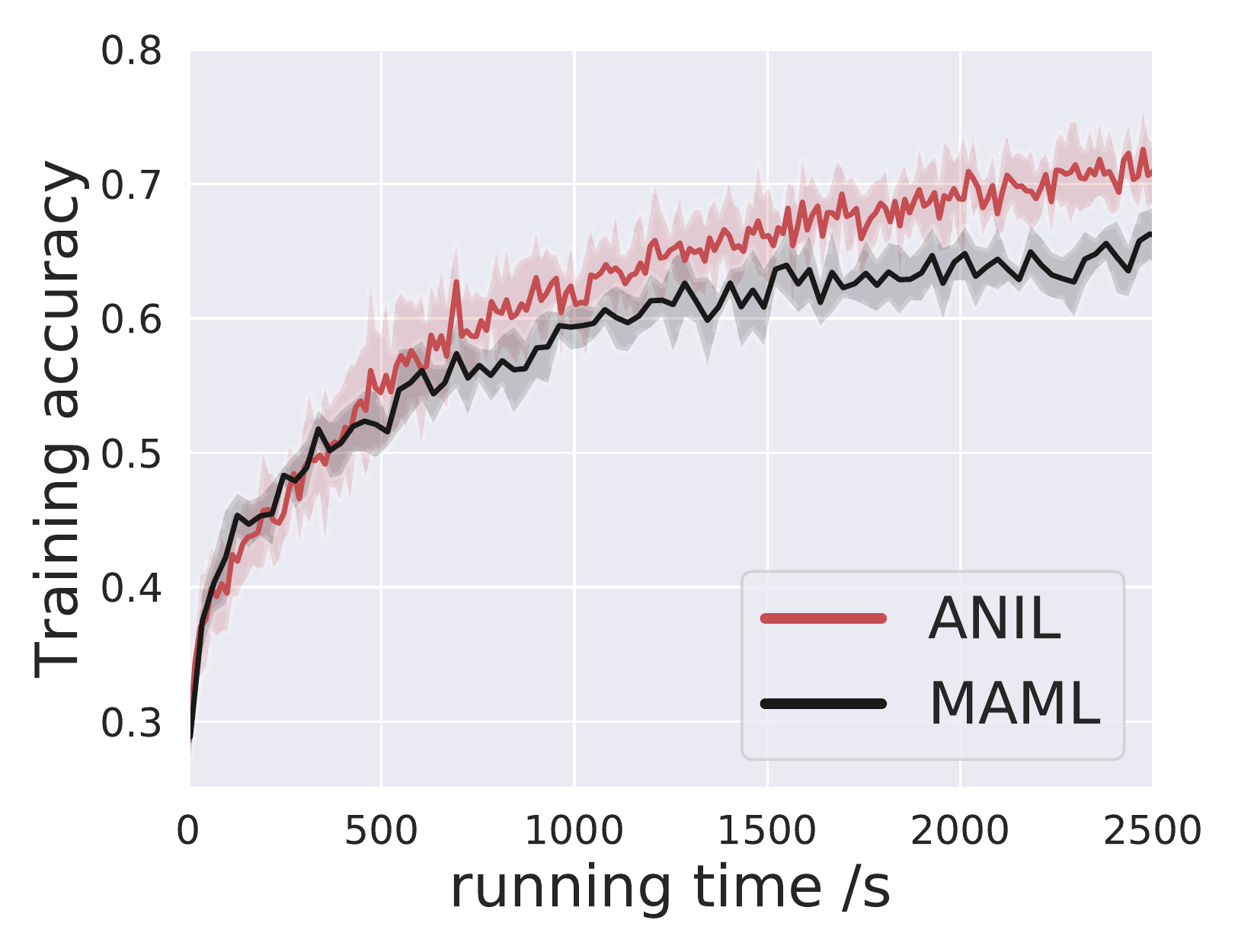}\includegraphics[width=34.5mm]{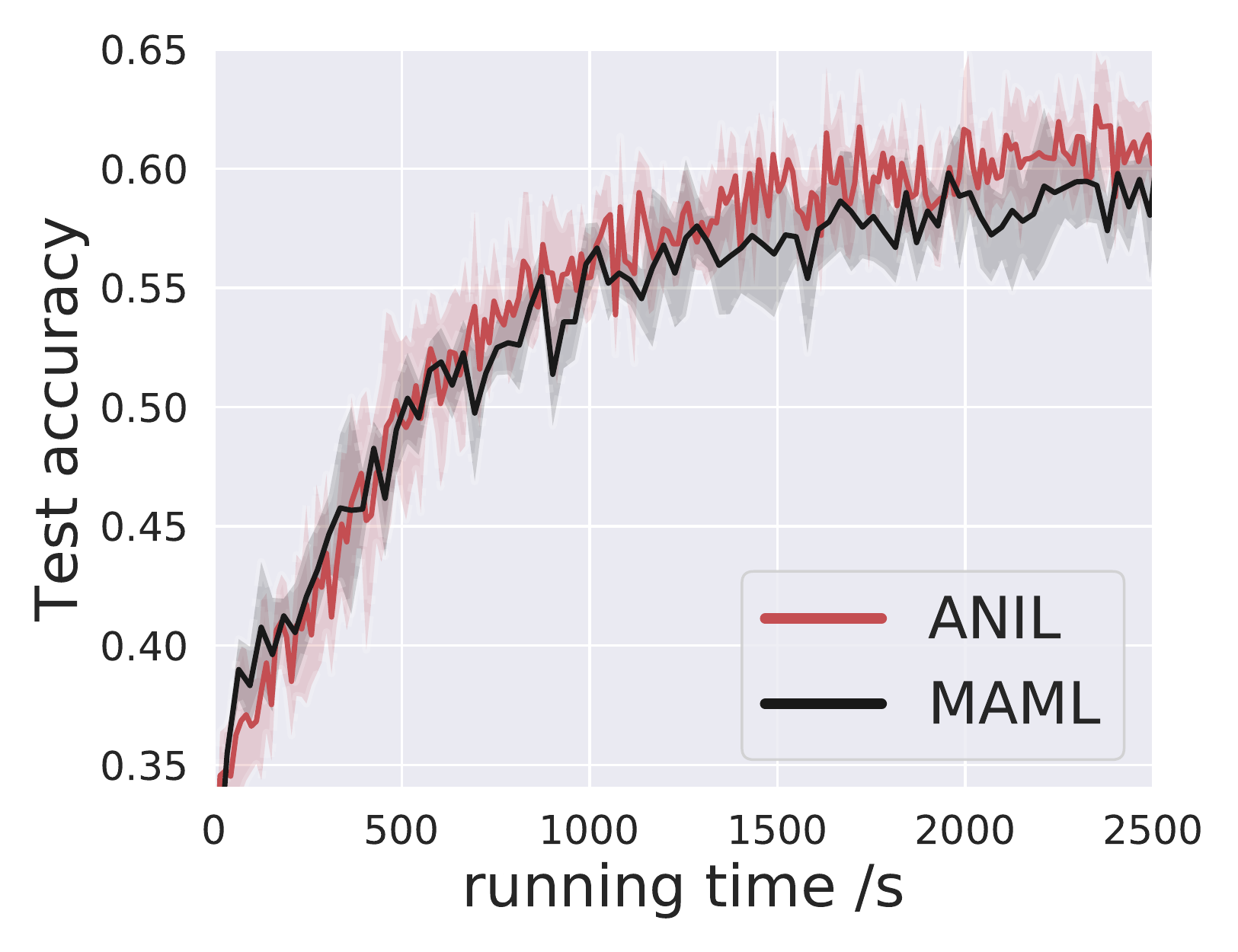}}  
	\caption{Computational comparison of ANIL and MAML.  For each dataset, left plot: training accuracy v.s. running time; right plot: test accuracy v.s. running time.}\label{figure:compps}
\end{figure*}


\section{Proof of~\Cref{le:gd_form} }
We first prove the form of the partial gradient $\frac{\partial L_{\mathcal{D}_i}( w^i_{k,N}, \phi_k)}{\partial w_k}$. Using the chain rule, we have 
\begin{align}\label{eq:form1}
\frac{\partial L_{\mathcal{D}_i}( w^i_{k,N}, \phi_k)}{\partial w_k} &= \frac{\partial w_{k,N}^i(w_k,\phi_k)}{\partial w_k} \nabla_{w} L_{\mathcal{D}_i} (w_{k,N}^i,\phi_k) + \frac{\partial \phi_k}{\partial w_k} \nabla_{\phi} L_{\mathcal{D}_i} (w_{k,N}^i,\phi_k) \nonumber
\\& = \frac{\partial w_{k,N}^i(w_k,\phi_k)}{\partial w_k} \nabla_{w} L_{\mathcal{D}_i} (w_{k,N}^i,\phi_k), 
\end{align}
where the last equality follows from the fact that $\frac{\partial \phi_k}{\partial w_k}  = 0$. Recall that the gradient updates in~\Cref{alg:anil} are given by 
\begin{align}\label{eq:gd_d2}
w_{k,m+1}^i = w_{k,m}^i - \alpha \nabla_{w} L_{\mathcal{S}_i} (w_{k,m}^i,\phi_k),\, m=0,1,...,N-1,
\end{align}
where $w_{k,0}^i=w_k $ for all $i$. 
Taking derivatives w.r.t. $w_k$ in~\cref{eq:gd_d2} yields
\begin{align}\label{eq:gd_up}
\frac{\partial w_{k,m+1}^i}{\partial w_k} = &\frac{\partial w_{k,m}^i}{\partial w_k} - \alpha \frac{\partial w_{k,m}^i}{\partial w_k}\nabla^2_{w} L_{\mathcal{S}_i}(w_{k,m}^i,\phi_k)-  \underbrace{\alpha \frac{\partial \phi_k}{\partial w_k}\nabla_\phi\nabla_{w} L_{\mathcal{S}_i}(w_{k,m}^i,\phi_k)}_{0}. 
\end{align}
Telescoping \cref{eq:gd_up} over $m$ from $0$ to $N-1$ yields
\begin{align*}
\frac{\partial w_{k,N}^i}{\partial w_k} = \prod_{m=0}^{N-1}(I - \alpha \nabla_w^2L_{\mathcal{S}_i}(w_{k,m}^i,\phi_k)),
\end{align*}
which, in conjunction~\cref{eq:form1}, yields the first part in~\Cref{le:gd_form}. 

For the second part, using chain rule, we have
\begin{align}\label{eq:sec}
\frac{\partial L_{\mathcal{D}_i}( w^i_{k,N}, \phi_k)}{\partial \phi_k} = \frac{\partial w^i_{k,N}}{\partial \phi_k}\nabla_w L_{\mathcal{D}_i}( w^i_{k,N}, \phi_k) + \nabla_\phi L_{\mathcal{D}_i}( w^i_{k,N}, \phi_k).
\end{align}
Taking derivates w.r.t. $\phi_k$ in \cref{eq:gd_d2} yields 
\begin{align*}
\frac{\partial w_{k,m+1}^i}{\partial \phi_k} = &\frac{\partial w_{k,m}^i}{\partial \phi_k} -\alpha \Big( \frac{\partial w_{k,m}^i}{\partial \phi_k}\nabla^2_w  L_{\mathcal{S}_i}(w_{k,m}^i,\phi_k) +\nabla_\phi \nabla_w L_{\mathcal{S}_i} (w_{k,m}^i,\phi_k)\Big) \nonumber
\\= &\frac{\partial w_{k,m}^i}{\partial \phi_k} (I - \alpha \nabla^2_w  L_{\mathcal{S}_i}(w_{k,m}^i,\phi_k)) - \alpha \nabla_\phi \nabla_w L_{\mathcal{S}_i} (w_{k,m}^i,\phi_k).
\end{align*}
Telescoping the above equality over $m$ from $0$ to $N-1$ yields
\begin{align*}
\frac{\partial w_{k,N}^i}{\partial \phi_k} = \frac{\partial w_{k,0}^i}{\partial \phi_k} &\prod_{m=0}^{N-1} (I - \alpha \nabla^2_w  L_{\mathcal{S}_i}(w_{k,m}^i,\phi_k) )  \nonumber
\\&- \alpha\sum_{m=0}^{N-1}\nabla_\phi\nabla_w L_{\mathcal{S}_i}(w_{k,m}^i,\phi_k)\prod_{j=m+1}^{N-1} (I-\alpha \nabla^2_{w} L_{\mathcal{S}_i}(w_{k,j}^i,\phi_k)),
\end{align*}
which, in conjunction with the fact that $\frac{\partial w_{k,0}^i}{\partial \phi_k}=\frac{\partial w_k}{\partial \phi_k}=0$ and~\cref{eq:sec}, yields the second part.

\section{Proof in~\Cref{se:strong-convex}: Strongly-Convex Inner Loop}
\subsection{Auxiliary Lemma}
The following lemma characterizes a bound on the difference between $w_{t}^{i}(w_1,\phi_1)$ and $w_{t}^{i}(w_2,\phi_2)$, where $w_{t}^{i}(w,\phi)$ corresponds to the $t^{th}$ inner-loop iteration starting from the initialization point $(w,\phi)$.
\begin{lemma}\label{le:support}
Choose $\alpha$ such that $1-2\alpha\mu+\alpha^2L^2>0$. Then, 
for any two points $(w_1,\phi_1),(w_2,\phi_2)\in\mathbb{R}^n$, we have 
\begin{align*}
\big\|w_{t}^{i}(w_1,\phi_1)- w_{t}^{i}(w_2,\phi_2)\big\| \leq (1-2\alpha\mu+\alpha^2L^2)^{\frac{t}{2}}\|w_1-w_2\| + \frac{\alpha L\|\phi_1-\phi_2\|}{1-\sqrt{1-2\alpha\mu+\alpha^2L^2}}.
\end{align*}
\end{lemma}
\begin{proof}
Based on the updates in~\cref{inner:gd}, we have 
\begin{align}
w_{m+1}^i(w_1,\phi_1)-&w_{m+1}^i(w_2,\phi_2) = w_{m}^i(w_1,\phi_1)-w_{m}^i(w_2,\phi_2)  \nonumber
\\&-\alpha \big(\nabla_w L_{\mathcal{S}_i}(w_m^i(w_1,\phi_1),\phi_1)-\nabla_w L_{\mathcal{S}_i}(w_m^i(w_2,\phi_2),\phi_1)\big)  \nonumber
\\&+ \alpha \big(\nabla_w L_{\mathcal{S}_i}(w_m^i(w_2,\phi_2),\phi_2)-\nabla_w L_{\mathcal{S}_i}(w_m^i(w_2,\phi_2),\phi_1)\big),\nonumber
\end{align}
which, together with the triangle inequality and~\Cref{assm:smooth}, yields
\begin{align}\label{eq:tri}
\|&w_{m+1}^i(w_1,\phi_1)-w_{m+1}^i(w_2,\phi_2) \|
\nonumber
\\&\leq \underbrace{\Big\|w_{m}^i(w_1,\phi_1)-w_{m}^i(w_2,\phi_2) -\alpha \big(\nabla_w L_{\mathcal{S}_i}(w_m^i(w_1,\phi_1),\phi_1)-\nabla_w L_{\mathcal{S}_i}(w_m^i(w_2,\phi_2),\phi_1)\big)\Big\|}_{P}  \nonumber
\\&\;\;\;\;+ \alpha L\|\phi_1-\phi_2\|.
\end{align} 
Our next step is to upper-bound the term $P$ in~\cref{eq:tri}. Note that 
\begin{align}\label{eq:p}
P^2 =& \|w_{m}^i(w_1,\phi_1)-w_{m}^i(w_2,\phi_2) \|^2 + \alpha^2\|\nabla_w L_{\mathcal{S}_i}(w_m^i(w_1,\phi_1),\phi_1)-\nabla_w L_{\mathcal{S}_i}(w_m^i(w_2,\phi_2),\phi_1)\|^2 \nonumber
\\&-2\alpha \Big\langle w_{m}^i(w_1,\phi_1)-w_{m}^i(w_2,\phi_2),  \nabla_w L_{\mathcal{S}_i}(w_m^i(w_1,\phi_1),\phi_1)-\nabla_w L_{\mathcal{S}_i}(w_m^i(w_2,\phi_2),\phi_1)  \Big\rangle \nonumber
\\\leq& (1+\alpha^2L^2-2\alpha\mu) \|w_{m}^i(w_1,\phi_1)-w_{m}^i(w_2,\phi_2) \|^2, 
\end{align}
where the last inequality follows from the strong-convexity of the loss function $L_{\mathcal{S}_i}(\cdot,\phi)$ that for any $w,w^\prime$ and $\phi$,
\begin{align*}
\langle w - w^\prime, \nabla_w L_{\mathcal{S}_i}(w,\phi) - \nabla_w L_{\mathcal{S}_i}(w^\prime,\phi)\rangle \geq  \mu \|w-w^\prime\|^2.
\end{align*}
Substituting~\cref{eq:p} into~\cref{eq:tri} yields  
\begin{align}
\|w_{m+1}^i(w_1,\phi_1)-w_{m+1}^i(w_2,\phi_2) \|\leq& \sqrt{1+\alpha^2L^2-2\alpha\mu}\|w_{m} ^i(w_1,\phi_1)-w_{m}^i(w_2,\phi_2) \| \nonumber
\\&+ \alpha L\|\phi_1-\phi_2\|.
\end{align}
Telescoping the above inequality over $m$ from $0$ to $t-1$ completes the proof.
\end{proof}
\subsection{Proof of~\Cref{le:strong-convex} }\label{append:str}
Using an approach similar to the proof of~\Cref{le:gd_form}, we have
\begin{align}\label{eq:ainiyo}
\frac{\partial L_{\mathcal{D}_i}( w^i_{N}, \phi)}{\partial w} =& \prod_{m=0}^{N-1}(I - \alpha \nabla_w^2L_{\mathcal{S}_i}(w_{m}^i,\phi)) \nabla_{w} L_{\mathcal{D}_i} (w_{N}^i,\phi). 
\end{align}
Let $w_m^i(w,\phi)$ denote the $m^{th}$ inner-loop iteration starting from $(w,\phi)$. Then, we have 
\begin{align}\label{eq:initss}
\Big\| &\frac{\partial L_{\mathcal{D}_i}( w^i_N,\phi)}{\partial w} \Big |_{(w_1,\phi_1)} -  \frac{\partial L_{\mathcal{D}_i}( w^i_N,\phi)}{\partial w} \Big |_{(w_2,\phi_2)} \Big\| \nonumber
\\ \leq & \underbrace{\Big\|\prod_{m=0}^{N-1}(I - \alpha \nabla_w^2L_{\mathcal{S}_i}(w_{m}^i(w_2,\phi_2),\phi_2)) \Big\|\Big\|\nabla_{w} L_{\mathcal{D}_i} (w_{N}^i(w_1,\phi_1),\phi_1)-\nabla_{w} L_{\mathcal{D}_i} (w_{N}^i(w_2,\phi_2),\phi_2)\Big\|}_{P} \nonumber
\\ &+\Big\|\prod_{m=0}^{N-1}(I - \alpha \nabla_w^2L_{\mathcal{S}_i}(w_{m}^i(w_1,\phi_1),\phi_1)) \nabla_{w} L_{\mathcal{D}_i} (w_{N}^i(w_1,\phi_1),\phi_1)\nonumber
\\&\hspace{1.2cm}\underbrace{\hspace{0.7cm}-\prod_{m=0}^{N-1}(I - \alpha \nabla_w^2L_{\mathcal{S}_i}(w_{m}^i(w_2,\phi_2),\phi_2)) \nabla_{w} L_{\mathcal{D}_i} (w_{N}^i(w_1,\phi_1),\phi_1)\Big\|}_{Q}, 
\end{align}
where $w_m^i(w,\phi)$ is obtained through the following gradient descent steps
\begin{align}\label{eq:updates}
w_{t+1}^i(w,\phi) = w_{t}^i(w,\phi) - \alpha \nabla_{w} L_{\mathcal{S}_i} (w_{t}^i(w,\phi),\phi),\, t=0,...,m-1\;\text{and} \;w_0^i(w,\phi)=w.
\end{align}
We next upper-bound the term $P$ in~\cref{eq:initss}. Based on the strongly-convexity of the function $L_{\mathcal{S}_i}(\cdot,\phi)$, we have $\big\|I-\alpha\nabla_w^2L_{\mathcal{S}_i}(\cdot,\phi)\big\|\leq 1-\alpha \mu$, and hence 
\begin{align}\label{eq:pss}
P \leq& (1-\alpha \mu)^{N}\big\|\nabla_{w} L_{\mathcal{D}_i} (w_{N}^i(w_1,\phi_1),\phi_1)-\nabla_{w} L_{\mathcal{D}_i} (w_{N}^i(w_2,\phi_2),\phi_2)\big\| \nonumber
\\\overset{(i)}\leq&(1-\alpha \mu)^{N} L\big(\|w_{N}^i(w_1,\phi_1)-w_{N}^i(w_2,\phi_2)\| +\|\phi_1-\phi_2\|\big) \nonumber
\\\overset{(ii)}\leq& (1-\alpha\mu)^{N}L\bigg( (1-2\alpha\mu+\alpha^2L^2)^{\frac{N}{2}}\|w_1-w_2\| + \frac{\alpha L\|\phi_1-\phi_2\|}{1-\sqrt{1-2\alpha\mu+\alpha^2L^2}} + \|\phi_1-\phi_2\|  \bigg) \nonumber
\\\overset{(iii)}\leq &  (1-\alpha\mu)^{\frac{3N}{2}} L \|w_1-w_2\| +  (1-\alpha\mu)^NL\left(\frac{2L}{\mu}+1\right)\|\phi_1-\phi_2\|,
\end{align}
where $(i)$ follows from~\Cref{assm:smooth}, (ii) follows from~\Cref{le:support}, and $(iii)$ follows from the fact that  $\alpha\mu = \frac{\mu^2}{L^2} =\alpha^2L^2$ and $\sqrt{1-x}\leq 1-\frac{1}{2}x$. 

To upper-bound the term $Q$ in~\cref{eq:initss}, we have
\begin{small}
\begin{align}\label{eq:qml}
Q\leq M\underbrace{\bigg\|\prod_{m=0}^{N-1}(I - \alpha \nabla_w^2L_{\mathcal{S}_i}(w_{m}^i(w_1,\phi_1),\phi_1)) -\prod_{m=0}^{N-1}(I - \alpha \nabla_w^2L_{\mathcal{S}_i}(w_{m}^i(w_2,\phi_2),\phi_2))  \bigg\|}_{P_{N-1}}.
\end{align}
\end{small}
\hspace{-0.12cm}To upper-bound $P_{N-1}$ in~\cref{eq:qml}, we define a more general quantity $P_t$ by replacing $N-1$ with $t$ in~\cref{eq:qml}. Using the triangle inequality, we have
\begin{align}\label{eq:telob}
P_t &\leq \alpha(1-\alpha\mu)^{t}\| \nabla_w^2L_{\mathcal{S}_i}(w_{t}^i(w_1,\phi_1),\phi_1)) -\nabla_w^2L_{\mathcal{S}_i}(w_{t}^i(w_2,\phi_2),\phi_2)) \| + (1-\alpha\mu) P_{t-1} \nonumber
\\\leq& (1-\alpha\mu) P_{t-1} + \alpha \rho(1-\alpha\mu)^{\frac{3t}{2}}  \|w_1-w_2\| +(1-\alpha\mu)^t\alpha\rho\left(\frac{2L}{\mu}+1\right)\|\phi_1-\phi_2\|.
\end{align}
Telescoping~\cref{eq:telob} over $t$ from $1$ to $N-1$ yields
\begin{align*}
P_{N-1}\leq &(1-\alpha\mu)^{N-1} P_0 + \sum_{t=1}^{N-1}  \alpha \rho(1-\alpha\mu)^{\frac{3t}{2}}  \|w_1-w_2\|(1-\alpha \mu)^{N-1-t} \nonumber
\\&+\sum_{t=1}^{N-1}(1-\alpha\mu)^t\alpha\rho\left(\frac{2L}{\mu}+1\right)\|\phi_1-\phi_2\|(1-\alpha \mu)^{N-1-t}, 
\end{align*}
which, in conjunction with $P_0\leq \alpha\rho(\|w_1-w_2\|+\|\phi_1-\phi_2\|)$, yields
\begin{align*}
P_{N-1}\leq& (1-\alpha\mu)^{N-1} \alpha\rho(\|w_1-w_2\|+\|\phi_1-\phi_2\|) + \alpha \rho \|w_1-w_2\| (1-\alpha\mu)^{N-1}\frac{\sqrt{1-\alpha\mu}}{1-\sqrt{1-\alpha\mu}} \nonumber
\\&+ \alpha\rho\left(\frac{2L}{\mu}+1\right)\|\phi_1-\phi_2\|(N-1)(1-\alpha\mu)^{N-1} \nonumber
\\\leq& \frac{2\rho}{\mu}(1-\alpha\mu)^{N-1}\|w_1-w_2\|+  \alpha\rho\left(\frac{2L}{\mu}+1\right)\|\phi_1-\phi_2\|N(1-\alpha\mu)^{N-1}, 
\end{align*}
which, in conjunction with \cref{eq:qml}, yields 
\begin{align}\label{eq:qss}
Q\leq  \frac{2\rho M}{\mu}(1-\alpha\mu)^{N-1}\|w_1-w_2\|+  \alpha\rho M\left(\frac{2L}{\mu}+1\right)\|\phi_1-\phi_2\|N(1-\alpha\mu)^{N-1}.
\end{align}
Substituting~\cref{eq:pss} and~\cref{eq:qss} into~\cref{eq:initss} yields
\begin{align}\label{eq:mamamiya}
\Big\| \frac{\partial L_{\mathcal{D}_i}( w^i_N,\phi)}{\partial w} \Big |_{(w_1,\phi_1)} &-  \frac{\partial L_{\mathcal{D}_i}( w^i_N,\phi)}{\partial w} \Big |_{(w_2,\phi_2)} \Big\| \nonumber
\\\leq& \Big((1-\alpha\mu)^{\frac{3N}{2}} L+ \frac{2\rho M}{\mu}(1-\alpha\mu)^{N-1}\Big) \|w_1-w_2\|  \nonumber
\\&+  \Big((1-\alpha\mu)^NL+\alpha\rho MN(1-\alpha\mu)^{N-1}\Big)\left(\frac{2L}{\mu}+1\right)\|\phi_1-\phi_2\|.
\end{align} 
Based on the definition $ L^{meta}( w,\phi)= \mathbb{E}_i L_{\mathcal{D}_i}( w^i_N,\phi)$ and using the Jensen's inequality, we have 
\begin{align}\label{eq:jensen}
 \Big\| \frac{\partial L^{meta}( w,\phi)}{\partial w} \big |_{(w_1,\phi_1)} -&  \frac{\partial L^{meta}( w,\phi)}{\partial w} \big |_{(w_2,\phi_2)} \Big\|  \nonumber
 \\&\leq \mathbb{E}_i  \Big\| \frac{\partial L_{\mathcal{D}_i}( w^i_N,\phi)}{\partial w} \Big |_{(w_1,\phi_1)} -  \frac{\partial L_{\mathcal{D}_i}( w^i_N,\phi)}{\partial w} \Big |_{(w_2,\phi_2)} \Big\|. 
\end{align}
Combining~\cref{eq:mamamiya} and~\cref{eq:jensen} completes the proof of the first item. 

We next prove the Lipschitz property of the partial gradient $\frac{\partial L_{\mathcal{D}_i}( w^i_N, \phi)}{\partial \phi}$. For notational convenience, we define several quantities below.
\begin{align}\label{eq:notations}
Q_m(w,\phi) &= \nabla_\phi\nabla_w L_{\mathcal{S}_i}(w_{m}^i(w,\phi),\phi), \;U_m(w,\phi)  = \prod_{j=m+1}^{N-1}(I-\alpha\nabla_w^2L_{\mathcal{S}_i}(w_{j}^i(w,\phi),\phi)), \nonumber
\\V_m(w,\phi) & = \nabla_w L_{\mathcal{D}_i}(w_N^i(w,\phi),\phi),
\end{align}
where we let $w_m^i(w,\phi)$ denote the $m^{th}$ inner-loop iteration starting from $(w,\phi)$.
Using an approach similar to the proof for~\Cref{le:gd_form}, we have
\begin{align}\label{eq:ainiyo11}
\frac{\partial L_{\mathcal{D}_i}( w^i_{N}, \phi)}{\partial \phi} =& -\alpha \sum_{m=0}^{N-1}\nabla_\phi\nabla_w L_{\mathcal{S}_i}(w_{m}^i,\phi) \prod_{j=m+1}^{N-1}(I-\alpha\nabla_w^2L_{\mathcal{S}_i}(w_{j}^i,\phi))\nabla_w L_{\mathcal{D}_i}(w_{N}^i,\phi) \nonumber
\\&+\nabla_\phi L_{\mathcal{D}_i}(w_{N}^i,\phi).
\end{align}
 Then, we have 
\begin{align}\label{eq:sect}
\Big\|&\frac{\partial L_{\mathcal{D}_i}( w^i_N, \phi)}{\partial \phi}\Big |_{(w_1,\phi_1)}  - \frac{\partial L_{\mathcal{D}_i}( w^i_N, \phi)}{\partial \phi} \Big |_{(w_2,\phi_2)}\Big\| \nonumber
\\&\;\leq \alpha\sum_{m=0}^{N-1}\|Q_m(w_1,\phi_1)U_m(w_1,\phi_1)V_m(w_1,\phi_1)-Q_m(w_2,\phi_2)U_m(w_2,\phi_2)V_m(w_2,\phi_2)\|  \nonumber
\\&\;\quad+\|\nabla_\phi L_{\mathcal{D}_i}(w_{N}^i(w_1,\phi_1),\phi_1)-\nabla_\phi L_{\mathcal{D}_i}(w_{N}^i(w_2,\phi_2),\phi_2)\|.
\end{align}
Using the triangle inequality, we have 
\begin{align}\label{R1R2R3} 
 \|Q_m(w_1,&\phi_1)U_m(w_1,\phi_1)V_m(w_1,\phi_1)-Q_m(w_2,\phi_2)U_m(w_2,\phi_2)V_m(w_2,\phi_2)\| \nonumber
 \\\leq &\underbrace{\|Q_m(w_1,\phi_1)-Q_m(w_2,\phi_2)\|\|U_m(w_1,\phi_1)\|\|V_m(w_1,\phi_1)\|}_{R_1}  \nonumber
 \\&+ \underbrace{\|Q_m(w_2,\phi_2)\|\|U_m(w_1,\phi_1)-U_m(w_2,\phi_2)\|\|V_m(w_1,\phi_1)\|}_{R_2}\nonumber
 \\&+\underbrace{\|Q_m(w_2,\phi_2)\|\|U_m(w_2,\phi_2)\|\|V_m(w_1,\phi_1)-V_m(w_2,\phi_2)\|}_{R_3}.
\end{align} 
Combining~\cref{eq:sect} and~\cref{R1R2R3}, we have  
\begin{align}\label{eq:initllls}
\Big\|&\frac{\partial L_{\mathcal{D}_i}( w^i_N, \phi)}{\partial \phi}\Big |_{(w_1,\phi_1)}  - \frac{\partial L_{\mathcal{D}_i}( w^i_N, \phi)}{\partial \phi} \Big |_{(w_2,\phi_2)}\Big\| \nonumber
\\&\;\;\leq \alpha\sum_{m=0}^{N-1}(R_1+R_2+R_3) +\|\nabla_\phi L_{\mathcal{D}_i}(w_{N}^i(w_1,\phi_1),\phi_1)-\nabla_\phi L_{\mathcal{D}_i}(w_{N}^i(w_2,\phi_2),\phi_2)\|.
\end{align}
To upper-bound $R_1$, we have 
\begin{align}\label{sR1}
R_1\leq& \tau (\|w_m^i(w_1,\phi_1)-w_m^i(w_2,\phi_2)\| + \|\phi_1-\phi_2\|) (1-\alpha \mu)^{N-m-1} M \nonumber
\\\leq& \tau M(1-\alpha\mu)^{N-\frac{m}{2}-1}  \|w_1-w_2\| + \tau M\Big(\frac{2L}{\mu}+1\Big) (1-\alpha \mu)^{N-m-1}\|\phi_1-\phi_2\|,
\end{align}
where the second inequality follows from~\Cref{le:support}. 
 
For $R_2$, based on Assumptions~\ref{assm:smooth} and~\ref{assm:second}, we have
\begin{align}\label{eq:rfist}
R_2\leq LM\|U_m(w_1,\phi_1)-U_m(w_2,\phi_2)\|.
\end{align}
Using the definitions of $U_m(w_1,\phi_1)$ and $U_m(w_2,\phi_2)$ in~\cref{eq:notations} and using the triangle inequality, we have 
\begin{align*}
\|U_m&(w_1,\phi_1)-U_m(w_2,\phi_2)\|  \nonumber
\\\leq& \alpha \|\nabla_w^2L_{\mathcal{S}_i}(w_{m+1}^i(w_1,\phi_1),\phi_1)-\nabla_w^2L_{\mathcal{S}_i}(w_{m+1}^i(w_2,\phi_2),\phi_2)\| \|U_{m+1}(w_1,\phi_1)\| \nonumber
\\&+ \|I-\alpha\nabla_w^2L_{\mathcal{S}_i}(w_{m+1}^i(w_1,\phi_1),\phi_1)\| \|U_{m+1}(w_1,\phi_1)-U_{m+1}(w_2,\phi_2)\| \nonumber
\\\leq&\alpha \rho (1-\alpha \mu)^{N-m-2}(\|w_{m+1}^i(w_1,\phi_1)-w_{m+1}^i(w_2,\phi_2)\| + \|\phi_1-\phi_2\|) \nonumber
\\&+ (1-\alpha \mu)\|U_{m+1}(w_1,\phi_1)-U_{m+1}(w_2,\phi_2)\| \nonumber
\\\leq&\alpha \rho (1-\alpha \mu)^{N-m-2}\Big((1-\alpha\mu)^{\frac{m+1}{2}}  \|w_1-w_2\| +  \Big(\frac{2L}{\mu}+1\Big)\|\phi_1-\phi_2\|\Big) \nonumber
\\&\hspace{3cm}+ (1-\alpha \mu)\|U_{m+1}(w_1,\phi_1)-U_{m+1}(w_2,\phi_2)\|,
\end{align*} 
where the last  inequality follows from~\Cref{le:support}.  
Telescoping the above inequality over $m$ yields
\begin{align}
\|&U_m(w_1,\phi_1)-U_m(w_2,\phi_2)\| \nonumber
\\&\leq (1-\alpha\mu)^{N-m-2}\|U_{N-2}(w_1,\phi_1)-U_{N-2}(w_2,\phi_2)\|  \nonumber
\\&\quad+\sum_{t=0}^{N-m-3} (1-\alpha\mu)^{t}\alpha \rho (1-\alpha \mu)^{N-m-t-2}\Big((1-\alpha\mu)^{\frac{m+t+1}{2}}  \|w_1-w_2\| +  \Big(\frac{2L}{\mu}+1\Big)\|\phi_1-\phi_2\|\Big),\nonumber
\end{align}
which, in conjunction with~\cref{eq:notations}, yields
\begin{align}\label{eq:ggsmida}
\|U_m(w_1,\phi_1)-U_m(w_2,\phi_2)\| \leq& \left( \frac{\alpha\rho}{1-\alpha\mu}+\frac{2\rho}{\mu}\right)(1-\alpha\mu)^{N-1-\frac{m}{2}}\|w_1-w_2\|  \nonumber
\\+&\alpha(N-1-m)\left( \rho+\frac{2\rho L}{\mu} \right)(1-\alpha\mu)^{N-2-m}\|\phi_1-\phi_2\|.
\end{align}
Combining~\cref{eq:rfist} and~\cref{eq:ggsmida} yields
\begin{align}\label{eq:r2bb}
R_2\leq& LM \left( \frac{\alpha\rho}{1-\alpha\mu}+\frac{2\rho}{\mu}\right)(1-\alpha\mu)^{N-1-\frac{m}{2}}\|w_1-w_2\|  \nonumber
\\&\quad+\alpha LM(N-1-m)\left( \rho+\frac{2\rho L}{\mu} \right)(1-\alpha\mu)^{N-2-m}\|\phi_1-\phi_2\|.
\end{align}
For $R_3$, using the triangle inequality, we have 
\begin{align}\label{eq:r3bb}
R_3\leq& L(1-\alpha \mu)^{N-m-1} L(\|w_N^i(w_1,\phi_1)-w_N^i(w_2,\phi_2)\|+\|\phi_1-\phi_2\|) \nonumber
\\\leq& L^2(1-\alpha\mu)^{\frac{3N}{2}-m-1}\|w_1-w_2\| + L^2\left(\frac{2L}{\mu}+1\right)(1-\alpha\mu)^{N-1-m}\|\phi_1-\phi_2\|.
\end{align}
where the last inequality follows from~\Cref{le:support}.

Combine $R_1,R_2$ and $R_3$ in~\cref{sR1},~\cref{eq:r2bb} and~\cref{eq:r3bb}, we have
\begin{align}\label{eq:youdianda}
\sum_{m=0}^{N-1}&(R_1+R_2+R_3) \leq  \frac{2\tau M}{\alpha\mu}(1-\alpha\mu)^{\frac{N-1}{2}}  \|w_1-w_2\| + \frac{\tau M}{\alpha\mu}\Big(\frac{2L}{\mu}+1\Big) \|\phi_1-\phi_2\| \nonumber
\\ &+ \frac{2LM}{\alpha\mu} \left( \frac{\alpha\rho}{1-\alpha\mu}+\frac{2\rho}{\mu}\right)(1-\alpha\mu)^{\frac{N-1}{2}}\|w_1-w_2\| +\frac{\alpha LM}{\alpha^2\mu^2}\left( \rho+\frac{2\rho L}{\mu} \right)\|\phi_1-\phi_2\| \nonumber
\\&+\frac{L^2}{\alpha\mu} (1-\alpha\mu)^{\frac{N}{2}}\|w_1-w_2\| + \frac{L^2}{\alpha\mu}\left(\frac{2L}{\mu}+1\right)\|\phi_1-\phi_2\|.
\end{align}
In addition, note that 
\begin{align}\label{eq:ggpopdasda}
 \|\nabla_\phi L_{\mathcal{D}_i}(w_{N}^i(w_1,\phi_1),&\phi_1)-\nabla_\phi L_{\mathcal{D}_i}(w_{N}^i(w_2,\phi_2),\phi_2)\| \nonumber
 \\&\leq (1-\alpha\mu)^{\frac{N}{2}} L \|w_1-w_2\| +  L\left(\frac{2L}{\mu}+1\right)\|\phi_1-\phi_2\|.
 \end{align}
 Combining~\cref{eq:initllls},~\cref{eq:youdianda}, and~\cref{eq:ggpopdasda} yields
\begin{align}\label{eq:opissas}
\Big\|\frac{\partial L_{\mathcal{D}_i}( w^i_N, \phi)}{\partial \phi}\Big |_{(w_1,\phi_1)} & - \frac{\partial L_{\mathcal{D}_i}( w^i_N, \phi)}{\partial \phi} \Big |_{(w_2,\phi_2)}\Big\| \nonumber
\\ \leq &  \left(  L+\frac{2\tau M}{\mu}+ \frac{2LM}{\mu} \left( \frac{\alpha\rho}{1-\alpha\mu}+\frac{2\rho}{\mu}\right) +\frac{L^2}{\mu} \right)(1-\alpha\mu)^{\frac{N-1}{2}}\|w_1-w_2\| \nonumber
\\&+ \left(L+\frac{\tau M}{\mu} + \frac{ LM\rho}{\mu^2}+\frac{L^2}{\mu} \right)\left(\frac{2L}{\mu}+1\right)\|\phi_1-\phi_2\|,
\end{align}
which, using an approach similar to~\cref{eq:jensen}, completes the proof. 
\subsection{Proof of~\Cref{th:strong-convex} }
For notational convenience, we define
\begin{align}\label{eq:ddfinetes}
g_{w}^i(k) &=  \frac{\partial L_{\mathcal{D}_i}( w^i_{k,N},\phi_k)}{\partial {w_k}},\quad g_{\phi}^i(k)  = \frac{\partial L_{\mathcal{D}_i}( w^i_{k,N},\phi_k)}{\partial {\phi_k}}, \nonumber
\\L_w&=(1-\alpha\mu)^{\frac{3N}{2}} L+ \frac{2\rho M}{\mu}(1-\alpha\mu)^{N-1},L_w^\prime= \Big(L+\alpha\rho MN\Big)(1-\alpha\mu)^{N-1}\left(\frac{2L}{\mu}+1\right),\nonumber
\\  L_\phi&= \left(  L+\frac{2\tau M}{\mu}+ \frac{2LM}{\mu} \left( \frac{\alpha\rho}{1-\alpha\mu}+\frac{2\rho}{\mu}\right) +\frac{L^2}{\mu} \right)(1-\alpha\mu)^{\frac{N-1}{2}}, \nonumber
\\L_\phi^\prime &=\left(L+\frac{\tau M}{\mu} + \frac{ LM\rho}{\mu^2}+\frac{L^2}{\mu} \right)\left(\frac{2L}{\mu}+1\right).
\end{align}
Then, the updates of~\Cref{alg:anil} are given by 
\begin{align}\label{alg:sgdpf}
w_{k+1}=  w_{k} - \frac{\beta_w}{B}\sum_{i\in\mathcal{B}_k} g_{w}^i(k)\,\text{ and }\, \phi_{k+1}= \phi_{k} - \frac{\beta_\phi}{B}\sum_{i\in\mathcal{B}_k}g_{\phi}^i(k).
\end{align}
Based on the smoothness properties established in~\cref{eq:mamamiya} and~\cref{eq:opissas} in the proof of~\Cref{le:strong-convex}, we have 
\begin{align*}
L^{meta}(w_{k+1},\phi_k) \leq & L^{meta}(w_k,\phi_k) + \left\langle \frac{\partial L^{meta}(w_k,\phi_k)}{\partial w_k}, w_{k+1}-w_k  \right\rangle + \frac{L_w}{2} \|w_{k+1}-w_k\|^2, \nonumber
\\ L^{meta}(w_{k+1},\phi_{k+1}) \leq & L^{meta}(w_{k+1},\phi_k) + \left\langle \frac{\partial L^{meta}(w_{k+1},\phi_k)}{\partial \phi_k}, \phi_{k+1}-\phi_k  \right\rangle + \frac{L^\prime_\phi}{2} \|\phi_{k+1}-\phi_k\|^2.
\end{align*}
Adding the above two inequalities, we have
\begin{align}\label{eq:gg1}
L^{meta}(w_{k+1},\phi_{k+1}) \leq&  L^{meta}(w_k,\phi_k) + \left\langle \frac{\partial L^{meta}(w_k,\phi_k)}{\partial w_k}, w_{k+1}-w_k  \right\rangle + \frac{L_w}{2} \|w_{k+1}-w_k\|^2 \nonumber
\\ &+\left\langle \frac{\partial L^{meta}(w_{k},\phi_k)}{\partial \phi_k}, \phi_{k+1}-\phi_k  \right\rangle + \frac{L_\phi^\prime}{2} \|\phi_{k+1}-\phi_k\|^2  \nonumber
\\&+\left\langle \frac{\partial L^{meta}(w_{k+1},\phi_k)}{\partial \phi_k}-\frac{\partial L^{meta}(w_{k},\phi_k)}{\partial \phi_k}, \phi_{k+1}-\phi_k  \right\rangle. 
\end{align}
Based on the Cauchy-Schwarz inequality, we have 
\begin{align}\label{eq:gg2}
\Big\langle \frac{\partial L^{meta}(w_{k+1},\phi_k)}{\partial \phi_k}-&\frac{\partial L^{meta}(w_{k},\phi_k)}{\partial \phi_k}, \phi_{k+1}-\phi_k  \Big\rangle  \nonumber
\\&\leq  L_\phi\|w_{k+1}-w_k\|\|\phi_{k+1}-\phi_k\|\nonumber
\\&\leq  \frac{L_\phi}{2}\|w_{k+1}-w_k\|^2 + \frac{L_\phi}{2}\|\phi_{k+1}-\phi_k\|^2.
\end{align}
Combining~\cref{eq:gg1} and~\cref{eq:gg2}, we have   
\begin{align*}
L^{meta}(w_{k+1},\phi_{k+1}) \leq&  L^{meta}(w_k,\phi_k) + \left\langle \frac{\partial L^{meta}(w_k,\phi_k)}{\partial w_k}, w_{k+1}-w_k  \right\rangle + \frac{L_w+L_\phi }{2}\|w_{k+1}-w_k\|^2 \nonumber
\\ &+\left\langle \frac{\partial L^{meta}(w_{k},\phi_k)}{\partial \phi_k}, \phi_{k+1}-\phi_k  \right\rangle + \frac{L_\phi+L_\phi^\prime}{2}\|\phi_{k+1}-\phi_k\|^2,  \nonumber
\end{align*}
which, in conjunction with the updates in~\cref{alg:sgdpf}, yields 
\begin{align}\label{eq:wellplay}
L^{meta}&(w_{k+1},\phi_{k+1}) \nonumber
\\\leq&  L^{meta}(w_k,\phi_k) - \left\langle \frac{\partial L^{meta}(w_k,\phi_k)}{\partial w_k}, \frac{\beta_w}{B}\sum_{i\in\mathcal{B}_k} g_{w}^i (k) \right\rangle + \frac{L_w+L_\phi }{2} \Big\| \frac{\beta_w}{B}\sum_{i\in\mathcal{B}_k} g_{w}^i (k)\Big\|^2 \nonumber
\\ &-\left\langle \frac{\partial L^{meta}(w_{k},\phi_k)}{\partial \phi_k},\frac{\beta_\phi}{B}\sum_{i\in\mathcal{B}_k}g_{\phi}^i(k)  \right\rangle + \frac{L_\phi+L_\phi^\prime}{2} \Big\|\frac{\beta_\phi}{B}\sum_{i\in\mathcal{B}_k}g_{\phi}^i(k)\Big\|^2.
\end{align}
Let $\mathbb{E}_k=\mathbb{E}(\cdot | w_k,\phi_k)$. Then, conditioning on $w_k,\phi_k$, and taking expectation over~\cref{eq:wellplay}, we have 
\begin{align}\label{eq:diedie}
\mathbb{E}_k L^{meta}(w_{k+1},\phi_{k+1}) \overset{(i)}\leq&  L^{meta}(w_k,\phi_k) - \beta_w\left\| \frac{\partial L^{meta}(w_k,\phi_k)}{\partial w_k}\right\|^2 + \frac{L_w+L_\phi}{2} \mathbb{E}_k\Big\| \frac{\beta_w}{B}\sum_{i\in\mathcal{B}_k} g_{w}^i(k) \Big\|^2 \nonumber
\\ &-\beta_\phi\left\| \frac{\partial L^{meta}(w_{k},\phi_k)}{\partial \phi_k}  \right\| + \frac{L_\phi+L_\phi^\prime}{2} \mathbb{E}_k\Big\|\frac{\beta_\phi}{B}\sum_{i\in\mathcal{B}_k}g_{\phi}^i(k)\Big\|^2 \nonumber
\\\leq & L^{meta}(w_k,\phi_k) - \beta_w\left\| \frac{\partial L^{meta}(w_k,\phi_k)}{\partial w_k}\right\|^2 + \frac{(L_w+L_\phi)\beta_w^2}{2B}\mathbb{E}_k\big\|g_{w}^i(k) \big\|^2 \nonumber
\\ &+ \frac{L_\phi+L_w}{2}\beta_w^2 \left\| \frac{\partial L^{meta}(w_k,\phi_k)}{\partial w_k}\right\|^2-\beta_\phi\left\| \frac{\partial L^{meta}(w_{k},\phi_k)}{\partial \phi_k}  \right\|^2 \nonumber
\\&+ \frac{L_\phi+L_\phi^\prime}{2} \left(\frac{\beta_\phi^2}{B}\mathbb{E}_k\big\|g_{\phi}^i(k) \big\|^2 + \beta_\phi^2 \left\| \frac{\partial L^{meta}(w_k,\phi_k)}{\partial \phi_k}\right\|^2\right),
\end{align}
where $(i)$ follows from the fact that $\mathbb{E}_k g_w^i(k)= \frac{\partial L^{meta}(w_k,\phi_k)}{\partial w_k}$ and $\mathbb{E}_k g_\phi^i(k)= \frac{\partial L^{meta}(w_k,\phi_k)}{\partial \phi_k}$.

 Our next
step is to upper-bound $\mathbb{E}_k\big\|g_{w}^i(k) \big\|^2$ and $\mathbb{E}_k\big\|g_{\phi}^i(k) \big\|^2$ in~\cref{eq:diedie}. Based on the definitions of $g_{w}^i(k)$ in~\cref{eq:ddfinetes} and using the explicit forms of the meta gradients in~\Cref{le:gd_form}, we have
\begin{align}\label{eq:omgs}
\mathbb{E}_k\big\|g_{w}^i(k) \big\|^2 \leq& \mathbb{E}_k\Big\|\prod_{m=0}^{N-1}(I - \alpha \nabla_w^2L_{\mathcal{S}_i}(w_{k,m}^i,\phi_k)) \nabla_{w} L_{\mathcal{D}_i} (w_{k,N}^i,\phi_k)\Big\|^2 \nonumber\\\leq &(1-\alpha\mu)^{2N} M^2.
\end{align}
Using an approach similar to~\cref{eq:omgs}, we have 
\begin{align}\label{eq:anothergg}
\mathbb{E}_k\big\|g_{\phi}^i(k) \big\|^2  \leq& 2\mathbb{E}_k\bigg\|\alpha \sum_{m=0}^{N-1}\nabla_\phi\nabla_w L_{\mathcal{S}_i}(w_{k,m}^i,\phi_k) \prod_{j=m+1}^{N-1}(I-\alpha\nabla_w^2L_{\mathcal{S}_i}(w_{k,j}^i,\phi_k))\nabla_w L_{\mathcal{D}_i}(w_{k,N}^i,\phi_k)\bigg\|^2 \nonumber
\\&+2\|\nabla_\phi L_{\mathcal{D}_i}(w_{k,N}^i,\phi_k)\|^2 \nonumber
\\ \leq& 2\alpha^2 L^2 M^2 \mathbb{E}_k\Big(\sum_{m=0}^{N-1} (1-\alpha \mu)^{N-1-m}\Big)^2 +2 M^2 \nonumber
\\< & \frac{2L^2M^2}{\mu^2}+2M^2 <  2M^2 \left( \frac{L^2}{\mu^2}+1 \right).
\end{align}
Substituting~\cref{eq:omgs} and~\cref{eq:anothergg} into~\cref{eq:diedie} yields
\begin{align}\label{eq:jingyi}
\mathbb{E}_k L^{meta}(w_{k+1},&\phi_{k+1}) \leq L^{meta}(w_k,\phi_k) - \left(\beta_w-\frac{L_w+L_\phi}{2}\beta_w^2\right)\left\| \frac{\partial L^{meta}(w_k,\phi_k)}{\partial w_k}\right\|^2 \nonumber
\\&+ \frac{(L_w+L_\phi)\beta_w^2}{2B}(1-\alpha\mu)^{2N} M^2-\left(\beta_\phi -\frac{L_\phi+L_\phi^\prime}{2}\beta_\phi^2 \right)\left\| \frac{\partial L^{meta}(w_{k},\phi_k)}{\partial \phi_k}  \right\|^2\nonumber
\\ & + \frac{ (L_\phi+L_\phi^\prime) \beta_\phi^2}{B} M^2 \left( \frac{L^2}{\mu^2}+1 \right).
\end{align}
Let $\beta_w=\frac{1}{L_w+L_\phi}$ and $\beta_\phi = \frac{1}{L_\phi+L_\phi^\prime}$. Then, 
unconditioning on $w_k$ and $\phi_k$ and telescoping~\cref{eq:jingyi} over $k$ from $0$ to $K-1$ yield 
\begin{align}
\frac{\beta_w}{2}&\frac{1}{K}\sum_{k=0}^{K-1}\mathbb{E}\left\| \frac{\partial L^{meta}(w_k,\phi_k)}{\partial w_k}\right\|^2 + \frac{\beta_\phi}{2}\frac{1}{K}\sum_{k=0}^{K-1}\mathbb{E}\left\| \frac{\partial L^{meta}(w_{k},\phi_k)}{\partial \phi_k}  \right\|^2 \nonumber
\\& \leq  \frac{L^{meta}(w_0,\phi_0)-\min_{w,\phi}L^{meta}(w,\phi) }{K}+ \frac{\beta_w}{2B}(1-\alpha\mu)^{2N} M^2+\frac{  \beta_\phi}{B} M^2 \left( \frac{L^2}{\mu^2}+1 \right).
\end{align}
Let $\Delta = L^{meta}(w_0,\phi_0)-\min_{w,\phi}L^{meta}(w,\phi)$ and let $\xi$ be chosen from $\{0,...,K-1\}$ uniformly at random. Then, we have 
\begin{align*}
\mathbb{E}\left\| \frac{\partial L^{meta}(w_\xi,\phi_\xi)}{\partial w_\xi}\right\|^2  \leq& \frac{2\Delta (L_w+L_\phi)}{K} + \frac{(1-\alpha\mu)^{2N} M^2}{B}+\frac{L_w+L_\phi}{L_\phi+L_\phi^\prime}\frac{2 }{B} M^2 \left( \frac{L^2}{\mu^2}+1 \right), \nonumber
\\\mathbb{E}\left\| \frac{\partial L^{meta}(w_\xi,\phi_\xi)}{\partial \phi_\xi}\right\|^2  \leq &\frac{2\Delta (L_\phi+L_\phi^\prime)}{K} +\frac{L_\phi+L_\phi^\prime}{L_w+L_\phi}\frac{1}{B}(1-\alpha\mu)^{2N} M^2+\frac{ 2}{B} M^2 \left( \frac{L^2}{\mu^2}+1 \right),
\end{align*}
which, in conjunction with the definitions of $L_\phi,L_\phi^\prime$ and $L_w$ in~\cref{eq:ddfinetes} and $\alpha=\frac{\mu}{L^2}$, yields 
\begin{align*}
\mathbb{E}\left\| \frac{\partial L^{meta}(w_\xi,\phi_\xi)}{\partial w_\xi}\right\|^2  \leq& \mathcal{O}\Bigg( \frac{  \frac{1}{\mu^2}\left(1-\frac{\mu^2}{L^2}\right)^{\frac{N}{2}}}{K}      +\frac{\frac{1}{\mu} \left(1-\frac{\mu^2}{L^2}\right)^{\frac{N}{2}}}{B}     \Bigg), \nonumber
\\\mathbb{E}\left\| \frac{\partial L^{meta}(w_\xi,\phi_\xi)}{\partial \phi_\xi}\right\|^2  \leq & \mathcal{O}\Bigg(\frac{ \frac{1}{\mu^2}\left(1-\frac{\mu^2}{L^2}\right)^{\frac{N}{2}}+\frac{1}{\mu^3}}{K} +\frac{\frac{1}{\mu}\left(1-\frac{\mu^2}{L^2}\right)^{\frac{3N}{2}}+\frac{1}{\mu^2}}{B}\Bigg).
\end{align*}
To achieve an $\epsilon$-stationary point, i.e., {\small $\mathbb{E}\left\| \frac{\partial L^{meta}(w,\phi)}{\partial w}\right\|^2 <\epsilon,\mathbb{E}\left\| \frac{\partial L^{meta}(w,\phi)}{\partial w}\right\|^2 <\epsilon$}, ANIL requires at most 
\begin{align*}
KBN=&\mathcal{O}\left(\frac{L^2}{\mu^2}\left(1-\frac{\mu^2}{L^2}\right)^{\frac{N}{2}}+\frac{L^3}{\mu^3}\right)\left(\frac{L}{\mu}\left(1-\frac{\mu^2}{L^2}\right)^{\frac{3N}{2}}+\frac{L^2}{\mu^2}\right)N\epsilon^{-2}
\\\leq&\mathcal{O}\left(\frac{N}{\mu^4}\left(1-\frac{\mu^2}{L^2}\right)^{\frac{N}{2}}+\frac{N}{\mu^5}\right)\epsilon^{-2}
\end{align*}
 gradient evaluations in $w$,  $KB=\mathcal{O}\Big(\mu^{-4}\left(1-\frac{\mu^2}{L^2}\right)^{N/2}+\mu^{-5}\Big)\epsilon^{-2}$ gradient evaluations in $\phi$,  and $KBN=\mathcal{O}\Big(\frac{N}{\mu^{4}}\left(1-\frac{\mu^2}{L^2}\right)^{N/2}+\frac{N}{\mu^{5}}\Big)\epsilon^{-2}$ evaluations of second-order derivatives.

\section{Proof in~\Cref{sec:nonconvex}: Nonconvex Inner Loop}
\subsection{Proof of~\Cref{le:smooth_nonconvex}}\label{appen:smooth_nonconvex}
Based on the explicit forms of the meta gradient in~\cref{eq:ainiyo} and using an approach similar to~\cref{eq:initss}, we have 
\begin{align}\label{eq:def}
 &\Big\| \frac{\partial L_{\mathcal{D}_i}( w^i_N,\phi)}{\partial w} \Big |_{(w_1,\phi_1)} -  \frac{\partial L_{\mathcal{D}_i}( w^i_N,\phi)}{\partial w} \Big |_{(w_2,\phi_2)} \Big\| \nonumber
\\ &= \Big\|\prod_{m=0}^{N-1}(I - \alpha \nabla_w^2L_{\mathcal{S}_i}(w_{m}^i(w_1,\phi_1),\phi_1)) \nabla_{w} L_{\mathcal{D}_i} (w_{N}^i(w_1,\phi_1),\phi_1)\nonumber
\\&\hspace{2cm}-\prod_{m=0}^{N-1}(I - \alpha \nabla_w^2L_{\mathcal{S}_i}(w_{m}^i(w_2,\phi_2),\phi_2)) \nabla_{w} L_{\mathcal{D}_i} (w_{N}^i(w_2,\phi_2),\phi_2)\Big\|, 
\end{align}
where $w_m^i(w,\phi)$ is obtained through the gradient descent steps in~\cref{eq:updates}.

Using the triangle inequality in~\cref{eq:def} yields
\begin{align}\label{eq:diff} 
\Big\| &\frac{\partial L_{\mathcal{D}_i}( w^i_N,\phi)}{\partial w} \Big |_{(w_1,\phi_1)} -  \frac{\partial L_{\mathcal{D}_i}( w^i_N,\phi)}{\partial w} \Big |_{(w_2,\phi_2)} \Big\| \nonumber
\\ \leq & \Big\|\prod_{m=0}^{N-1}(I - \alpha \nabla_w^2L_{\mathcal{S}_i}(w_{m}^i(w_2,\phi_2),\phi_2)) \Big\|\Big\|\nabla_{w} L_{\mathcal{D}_i} (w_{N}^i(w_1,\phi_1),\phi_1)-\nabla_{w} L_{\mathcal{D}_i} (w_{N}^i(w_2,\phi_2),\phi_2)\Big\| \nonumber
\\ &+\Big\|\prod_{m=0}^{N-1}(I - \alpha \nabla_w^2L_{\mathcal{S}_i}(w_{m}^i(w_1,\phi_1),\phi_1)) \nabla_{w} L_{\mathcal{D}_i} (w_{N}^i(w_1,\phi_1),\phi_1)\nonumber
\\&\hspace{2cm}-\prod_{m=0}^{N-1}(I - \alpha \nabla_w^2L_{\mathcal{S}_i}(w_{m}^i(w_2,\phi_2),\phi_2)) \nabla_{w} L_{\mathcal{D}_i} (w_{N}^i(w_1,\phi_1),\phi_1)\Big\|. 
\end{align}
Our next two steps are to upper-bound the two terms at the right hand side of~\cref{eq:diff}, respectively.

{Step 1: Upper-bound the first term at the right hand side of~\cref{eq:diff}.}
\begin{align}\label{eq:first}
&\Big\|\prod_{m=0}^{N-1}(I - \alpha \nabla_w^2L_{\mathcal{S}_i}(w_{m}^i(w_2,\phi_2),\phi_2)) \Big\|\Big\|\nabla_{w} L_{\mathcal{D}_i} (w_{N}^i(w_1,\phi_1),\phi_1)-\nabla_{w} L_{\mathcal{D}_i} (w_{N}^i(w_2,\phi_2),\phi_2)\Big\| \nonumber
\\&\overset{(i)}\leq (1+\alpha L)^N \Big\|\nabla_{w} L_{\mathcal{D}_i} (w_{N}^i(w_1,\phi_1),\phi_1)-\nabla_{w} L_{\mathcal{D}_i} (w_{N}^i(w_2,\phi_2),\phi_2)\Big\|\nonumber
\\&\overset{(ii)}\leq (1+\alpha L)^N L(\|w_{N}^i(w_1,\phi_1)-w_{N}^i(w_2,\phi_2)\|+\|\phi_1-\phi_2\|),
\end{align}
  where $(i)$ follows from the fact that $\| \nabla_w^2L_{\mathcal{S}_i}(w_{m}^i(w_2,\phi_2),\phi_2)\|\leq L$, and $(ii)$ follows from~\Cref{assm:smooth}. Based on the gradient descent steps in~\cref{eq:updates}, we have, for any $0\leq m\leq N-1$,  
  \begin{align}
 & w_{m+1}^i(w_1,\phi_1) - w_{m+1}^i(w_2,\phi_2)  \nonumber
 \\ &\quad= w_{m}^i(w_1,\phi_1) -w_{m}^i(w_2,\phi_2) - \alpha\big(\nabla_{w} L_{\mathcal{S}_i} (w_{m}^i(w_1,\phi_1),\phi_1) -\nabla_{w} L_{\mathcal{S}_i} (w_{m}^i(w_2,\phi_2),\phi_2)\big).\nonumber
  \end{align}
Based on the above equality, we further obtain 
\begin{align}
\|w_{m+1}^i(w_1,\phi_1) - w_{m+1}^i(w_2,\phi_2)\|\leq& \|w_{m}^i(w_1,\phi_1) -w_{m}^i(w_2,\phi_2) \| \nonumber
\\&+ \alpha \|\nabla_{w} L_{\mathcal{S}_i} (w_{m}^i(w_1,\phi_1),\phi_1) -\nabla_{w} L_{\mathcal{S}_i} (w_{m}^i(w_2,\phi_2),\phi_2)\| \nonumber 
\\\leq&(1+\alpha L)\|w_{m}^i(w_1,\phi_1) -w_{m}^i(w_2,\phi_2) \|  + \alpha L \|\phi_1-\phi_2\|,\nonumber
\end{align}
where the last inequality follows from~\Cref{assm:smooth}.
Telescoping the above inequality over $m$ from $0$ to $N-1$ yields
\begin{align}\label{eq:wn}
\|w_{N}^i(w_1,\phi_1) - w_{N}^i(w_2,\phi_2)\| \leq (1+\alpha L)^N \|w_1-w_2\| + ((1+\alpha L)^N-1)\|\phi_1-\phi_2\|.
\end{align}
Combining~\cref{eq:first} and~\cref{eq:wn} yields
\begin{align}\label{upb1}
&\Big\|\prod_{m=0}^{N-1}(I - \alpha \nabla_w^2L_{\mathcal{S}_i}(w_{m}^i(w_2,\phi_2),\phi_2)) \Big\|\Big\|\nabla_{w} L_{\mathcal{D}_i} (w_{N}^i(w_1,\phi_1),\phi_1)-\nabla_{w} L_{\mathcal{D}_i} (w_{N}^i(w_2,\phi_2),\phi_2)\Big\| \nonumber
\\&\leq (1+\alpha L)^{2N} L(\|w_1-w_2\|+\|\phi_1-\phi_2\|).
\end{align}

{ Step 2: Upper-bound the second term at the right hand side of~\cref{eq:diff}.}

Based on item 2 in~\Cref{assm:smooth}, we have that $\|\nabla_{w} L_{\mathcal{D}_i}(\cdot,\cdot)\|\leq M$. Then, the second term at the right hand side of~\cref{eq:diff} is further upper-bounded by 
\begin{align}\label{m:quantity}
M\underbrace{\bigg\|\prod_{m=0}^{N-1}(I - \alpha \nabla_w^2L_{\mathcal{S}_i}(w_{m}^i(w_1,\phi_1),\phi_1)) -\prod_{m=0}^{N-1}(I - \alpha \nabla_w^2L_{\mathcal{S}_i}(w_{m}^i(w_2,\phi_2),\phi_2))\bigg\|}_{P_{N-1}}. 
\end{align}
In order to upper-bound $P_{N-1}$ in~\cref{m:quantity}, we define a more general quantity $P_t$ by replacing $N-1$ with $t$ in~\cref{m:quantity}. Based on the triangle inequality, we have 
\begin{align}
P_t \leq& \alpha \bigg\|\prod_{m=0}^{t-1}(I - \alpha \nabla_w^2L_{\mathcal{S}_i}(w_{m}^i,\phi_1))\bigg\|\Big\|  \nabla_w^2L_{\mathcal{S}_i}(w_{t}^i(w_1,\phi_1),\phi_1) - \nabla_w^2L_{\mathcal{S}_i}(w_{t}^i(w_2,\phi_2),\phi_2)\Big\| \nonumber
\\ &+ P_{t-1} \Big\|I - \alpha \nabla_w^2L_{\mathcal{S}_i}(w_{t}^i(w_2,\phi_2),\phi_2)\Big\| \nonumber
\\\leq & \alpha(1+\alpha L)^t \rho (\|w_{t}^i(w_1,\phi_1) -w_{t}^i(w_2,\phi_2)\| + \|\phi_1-\phi_2\|) + (1+\alpha L)P_{t-1}\nonumber
\\\overset{(i)}\leq &\alpha \rho(1+\alpha L)^{2t}  (\|w_1-w_2\|+\|\phi_1-\phi_2\|) + (1+\alpha L) P_{t-1}, \nonumber
\end{align} 
where $(i)$ follows from~\cref{eq:wn}.  Rearranging the above inequality, we have  
\begin{align}\label{eq:idc}
P_t - &\frac{\rho}{L}(1+\alpha L)^{2t+1}  (\|w_1-w_2\|+\|\phi_1-\phi_2\|)  \nonumber
\\&\leq (1+\alpha L)(P_{t-1}-\frac{\rho}{L}(1+\alpha L)^{2t-1}(\|w_1-w_2\|+\|\phi_1-\phi_2\|)).
\end{align}
Telescoping \cref{eq:idc} over $t$ from $1$ to $N-1$ yields
\begin{align}
P_{N-1} - \frac{\rho}{L}(1+\alpha L)^{2N-1}  &(\|w_1-w_2\|+\|\phi_1-\phi_2\|)   \nonumber
\\&\leq (1+\alpha L)^{N}\Big(P_{0}-\frac{\rho}{L}(1+\alpha L)(\|w_1-w_2\|+\|\phi_1-\phi_2\|) \Big),  \nonumber
\end{align}
which, in conjunction with $P_{0}=\alpha\|\nabla_w^2L_{\mathcal{S}_i}(w_1,\phi_1)-\nabla_w^2L_{\mathcal{S}_i}(w_2,\phi_2)\|\leq \alpha\rho(\|w_1-w_2\|+\|\phi_1-\phi_2\|)$, yields
\begin{align}\label{eq:pn1}
P_{N-1} - \frac{\rho}{L}(1+\alpha L)^{2N-1}  &(\|w_1-w_2\|+\|\phi_1-\phi_2\|)   \nonumber
\\&\leq (1+\alpha L)^{N}\Big(\frac{\rho}{L}(\|w_1-w_2\|+\|\phi_1-\phi_2\|) \Big)\nonumber
\\&\leq \frac{\rho}{L}(1+\alpha L)^{2N-1}  (\|w_1-w_2\|+\|\phi_1-\phi_2\|), 
 \end{align}
 where the last inequality follows because $N\geq 1$. 
 Combining \cref{m:quantity}, and~\cref{eq:pn1}, we have that the second term at the right hand side of~\cref{eq:diff} is upper-bounded by 
 \begin{align}\label{upb2}
\frac{ 2M\rho}{L}(1+\alpha L)^{2N-1}  &(\|w_1-w_2\|+\|\phi_1-\phi_2\|).
 \end{align}
 {Step 3: Combine two bounds in Steps 1 and 2.}
 
 Combining~\cref{upb1},~\cref{upb2}, and using $\alpha <\mathcal{O}(\frac{1}{N})$, we have 
 \begin{align}\label{eq:j1}
 \Big\| \frac{\partial L_{\mathcal{D}_i}( w^i_N,\phi)}{\partial w} &\Big |_{(w_1,\phi_1)} -  \frac{\partial L_{\mathcal{D}_i}( w^i_N,\phi)}{\partial w} \Big |_{(w_2,\phi_2)} \Big\| \nonumber
\\ &\leq \Big(1+\alpha L+\frac{2M\rho}{L}\Big) (1+\alpha L)^{2N-1} L(\|w_1-w_2\|+\|\phi_1-\phi_2\|) \nonumber
\\&\leq \text{poly}(M,\rho,\alpha,L) N(\|w_1-w_2\|+\|\phi_1-\phi_2\|),
\end{align}
which, using an approach similar to~\cref{eq:jensen}, completes the proof of the first item in~\Cref{le:smooth_nonconvex}.
We next prove the Lipschitz property of the partial gradient $\frac{\partial L_{\mathcal{D}_i}( w^i_N, \phi)}{\partial \phi}$. Using an approach similar to~\cref{eq:sect} and~\cref{R1R2R3}, we have 
\begin{align}\label{eq:qopqop}
\Big\|&\frac{\partial L_{\mathcal{D}_i}( w^i_N, \phi)}{\partial \phi}\Big |_{(w_1,\phi_1)}  - \frac{\partial L_{\mathcal{D}_i}( w^i_N, \phi)}{\partial \phi} \Big |_{(w_2,\phi_2)}\Big\| \nonumber
\\&\leq \alpha\sum_{m=0}^{N-1}(R_1+R_2+R_3) +\|\nabla_\phi L_{\mathcal{D}_i}(w_{N}^i(w_1,\phi_1),\phi_1)-\nabla_\phi L_{\mathcal{D}_i}(w_{N}^i(w_2,\phi_2),\phi_2)\|,
\end{align}
where $R_1,R_2$ and $R_3$ are defined in~\cref{R1R2R3}.

To upper-bound $R_1$ in the above inequality, we have
\begin{align}\label{R1}
R_1\overset{(i)}\leq &\tau (\|w_m^i(w_1,\phi_1)-w_m^i(w_2,\phi_2)\| + \|\phi_1-\phi_2\|) (1+\alpha L)^{N-m-1} M \nonumber
\\\overset{(ii)}\leq& \tau M (1+\alpha L)^{N-1} (\|w_1-w_2\| + \|\phi_1-\phi_2\|),
\end{align}
where $(i)$ follows from Assumptions~\ref{assm:smooth} and~\ref{assm:second} and $(ii)$ follows from~\cref{eq:wn}.

 For $R_2$, using the triangle inequality, we have
\begin{align}\label{eq:impor}
\|U_m&(w_1,\phi_1)-U_m(w_2,\phi_2)\|  \nonumber
\\\leq& \alpha \|\nabla_w^2L_{\mathcal{S}_i}(w_{m+1}^i(w_1,\phi_1),\phi_1)-\nabla_w^2L_{\mathcal{S}_i}(w_{m+1}^i(w_2,\phi_2),\phi_2)\| \|U_{m+1}(w_1,\phi_1)\| \nonumber
\\&+ \|I-\alpha\nabla_w^2L_{\mathcal{S}_i}(w_{m+1}^i(w_1,\phi_1),\phi_1)\| \|U_{m+1}(w_1,\phi_1)-U_{m+1}(w_2,\phi_2)\| \nonumber
\\\leq&\alpha \rho (1+\alpha L)^{N-m-2}(\|w_{m+1}^i(w_1,\phi_1)-w_{m+1}^i(w_2,\phi_2)\| + \|\phi_1-\phi_2\|) \nonumber
\\&+ (1+\alpha L)\|U_{m+1}(w_1,\phi_1)-U_{m+1}(w_2,\phi_2)\| \nonumber
\\\leq&\alpha \rho (1+\alpha L)^{N-1}(\|w_1-w_2\| + \|\phi_1-\phi_2\|) \nonumber
\\&+ (1+\alpha L)\|U_{m+1}(w_1,\phi_1)-U_{m+1}(w_2,\phi_2)\|.
\end{align}
Telescoping the above inequality over $m$ yields
\begin{align}
&\|U_m(w_1,\phi_1)-U_m(w_2,\phi_2)\| +\frac{\rho}{L} (1+\alpha L)^{N-1}(\|w_1-w_2\| + \|\phi_1-\phi_2\|)\nonumber
 \\&\leq (1+\alpha L)^{N-m-2}\Big(\|U_{N-2}(w_1,\phi_1)-U_{N-2}(w_2,\phi_2)\|+\frac{\rho}{L} (1+\alpha L)^{N-1}(\|w_1-w_2\| + \|\phi_1-\phi_2\|)\Big), \nonumber
\end{align}
which, in conjunction with 
\begin{align*}
\|U_{N-2}(w_1,\phi_1)-U_{N-2}(w_2,\phi_2)\|=& \alpha\|\nabla_w^2L_{\mathcal{S}_i}(w_{N-1}^i(w_1,\phi_1),\phi_1)-\nabla_w^2L_{\mathcal{S}_i}(w_{N-1}^i(w_2,\phi_2),\phi_2)\| \nonumber
\\\leq& \alpha\rho (1+\alpha L)^{N-1}(\|w_1-w_2\|+\|\phi_1-\phi_2\|),
\end{align*}
yields that
\begin{align}\label{eq:um1}
\|U_m(w_1,\phi_1)-U_m(w_2,\phi_2)\| \leq& \big(\alpha \rho+\frac{\rho}{L}\big)(1+\alpha L)^{2N-m-3}(\|w_1-w_2\|+\|\phi_1-\phi_2\|) \nonumber
\\& -\frac{\rho}{L}(1+\alpha L)^{N-1} (\|w_1-w_2\|+\|\phi_1-\phi_2\|).
\end{align}
Based on~\Cref{assm:smooth}, we have $\|Q_m(w_2,\phi_2)\|\leq L$ and $\|V_m(w_1,\phi_1)\|\leq M$, which, combined with \cref{eq:um1} and the definition of $R_2$  in~\cref{R1R2R3}, yields
\begin{align}\label{R2}
R_2\leq& ML\Big(\alpha \rho+\frac{\rho}{L}\Big)(1+\alpha L)^{2N-m-3}(\|w_1-w_2\|+\|\phi_1-\phi_2\|)\nonumber
\\&-M\rho (1+\alpha L)^{N-1} (\|w_1-w_2\|+\|\phi_1-\phi_2\|).
\end{align}

For  $R_3$, using~\Cref{assm:smooth}, we have 
\begin{align}\label{R3}
R_3\leq& L (1+\alpha L)^{N-m-1} \| \nabla_w L_{\mathcal{D}_i}(w_N^i(w_1,\phi_1),\phi_1)- \nabla_w L_{\mathcal{D}_i}(w_N^i(w_2,\phi_2),\phi_2)\| \nonumber
\\\leq & L^2 (1+\alpha L)^{2N-m-1} (\|w_1-w_2\|+\|\phi_1-\phi_2\|),  
\end{align}
where the last inequality follows from \cref{eq:wn}. Combining~\cref{R1},~\cref{R2} and~\cref{R3} yields
\begin{align}\label{eq:quv}
 R_1+R_2+R_3\leq & M(\tau-\rho) (1+\alpha L)^{N-1} (\|w_1-w_2\| + \|\phi_1-\phi_2\|)  \nonumber
 \\&+M\rho(1+\alpha L)^{2N-m-2}(\|w_1-w_2\|+\|\phi_1-\phi_2\|)\nonumber
 \\&+L^2 (1+\alpha L)^{2N-m-1} (\|w_1-w_2\|+\|\phi_1-\phi_2\|).  
\end{align}
 Combining~\cref{eq:qopqop},~\cref{eq:quv}, and using~\cref{eq:wn} and $\alpha <\mathcal{O}(\frac{1}{N})$, we have 
 \begin{align}\label{eq:winnsa}
 \Big\|&\frac{\partial L_{\mathcal{D}_i}( w^i_N, \phi)}{\partial \phi}\Big |_{(w_1,\phi_1)}  - \frac{\partial L_{\mathcal{D}_i}( w^i_N, \phi)}{\partial \phi} \Big |_{(w_2,\phi_2)}\Big\| \nonumber
 \\&\leq \Big(\alpha M(\tau-\rho)  N(1+\alpha L)^{N-1} + \Big(L+\frac{\rho M}{L}\Big)(1+\alpha L)^{2N}\Big)(\|w_1-w_2\|+\|\phi_1-\phi_2\|)\nonumber
\\&\leq\text{\normalfont poly}(M,\rho,\tau,\alpha,L) N(\|w_1-w_2\| + \|\phi_1-\phi_2\|),
 \end{align}
 which, using an approach similar to~\cref{eq:jensen}, finishes the proof of the second item in~\Cref{le:smooth_nonconvex}.
\subsection{Proof of \Cref{th:nonconvex}}
For notational convenience, we define 
\begin{align}\label{eq:definitions}
g_{w}^i(k) &=  \frac{\partial L_{\mathcal{D}_i}( w^i_{k,N},\phi_k)}{\partial {w_k}},\quad g_{\phi}^i(k)  = \frac{\partial L_{\mathcal{D}_i}( w^i_{k,N},\phi_k)}{\partial {\phi_k}}, \nonumber
\\ L_w &=\big(L+\alpha L^2+2M\rho\big) (1+\alpha L)^{2N-1}, \nonumber
\\L_\phi & = \alpha M(\tau-\rho)  N(1+\alpha L)^{N-1} + \left(L+\frac{\rho M}{L}\right)(1+\alpha L)^{2N}.
\end{align}
Based on the smoothness properties established in~\cref{eq:j1} and~\cref{eq:winnsa} in the proof of~\Cref{le:smooth_nonconvex}, we have 
\begin{align*}
L^{meta}(w_{k+1},\phi_k) \leq & L^{meta}(w_k,\phi_k) + \left\langle \frac{\partial L^{meta}(w_k,\phi_k)}{\partial w_k}, w_{k+1}-w_k  \right\rangle + \frac{L_w}{2} \|w_{k+1}-w_k\|^2, \nonumber
\\ L^{meta}(w_{k+1},\phi_{k+1}) \leq & L^{meta}(w_{k+1},\phi_k) + \left\langle \frac{\partial L^{meta}(w_{k+1},\phi_k)}{\partial \phi_k}, \phi_{k+1}-\phi_k  \right\rangle + \frac{L_\phi}{2} \|\phi_{k+1}-\phi_k\|^2.
\end{align*}
Adding the above two inequalities, and using an approach similar to~\cref{eq:wellplay}, we have 
\begin{align}\label{eq:noexp}
L^{meta}&(w_{k+1},\phi_{k+1}) \nonumber
\\\leq&  L^{meta}(w_k,\phi_k) - \left\langle \frac{\partial L^{meta}(w_k,\phi_k)}{\partial w_k}, \frac{\beta_w}{B}\sum_{i\in\mathcal{B}_k} g_{w}^i (k) \right\rangle + \frac{L_w+L_\phi }{2} \Big\| \frac{\beta_w}{B}\sum_{i\in\mathcal{B}_k} g_{w}^i (k)\Big\|^2 \nonumber
\\ &-\left\langle \frac{\partial L^{meta}(w_{k},\phi_k)}{\partial \phi_k},\frac{\beta_\phi}{B}\sum_{i\in\mathcal{B}_k}g_{\phi}^i(k)  \right\rangle + L_\phi \Big\|\frac{\beta_\phi}{B}\sum_{i\in\mathcal{B}_k}g_{\phi}^i(k)\Big\|^2.
\end{align}
Let $\mathbb{E}_k=\mathbb{E}(\cdot | w_k,\phi_k)$. Then, conditioning on $w_k,\phi_k$, taking expectation over~\cref{eq:noexp} and using an approach similar to~\cref{eq:diedie}, we have 
\begin{align}\label{eq:finals}
\mathbb{E}_k L^{meta}(w_{k+1},\phi_{k+1})\leq&  L^{meta}(w_k,\phi_k) - \beta_w\left\| \frac{\partial L^{meta}(w_k,\phi_k)}{\partial w_k}\right\|^2 + \frac{(L_w+L_\phi)\beta_w^2}{2B}\mathbb{E}_k\big\|g_{w}^i(k) \big\|^2 \nonumber
\\ &+ \frac{L_\phi+L_w}{2}\beta_w^2 \left\| \frac{\partial L^{meta}(w_k,\phi_k)}{\partial w_k}\right\|^2-\beta_\phi\left\| \frac{\partial L^{meta}(w_{k},\phi_k)}{\partial \phi_k}  \right\|^2 \nonumber
\\&+L_\phi\left(\frac{\beta_\phi^2}{B}\mathbb{E}_k\big\|g_{\phi}^i(k) \big\|^2 + \beta_\phi^2 \left\| \frac{\partial L^{meta}(w_k,\phi_k)}{\partial \phi_k}\right\|^2\right).
\end{align}
 
 Our next
step is to upper-bound $\mathbb{E}_k\big\|g_{w}^i(k) \big\|^2$ and $\mathbb{E}_k\big\|g_{\phi}^i(k) \big\|^2$ in~\cref{eq:finals}. Based on the definitions of $g_{w}^i(k)$ in~\cref{eq:definitions} and~\Cref{le:gd_form}, we have
\begin{align}\label{eq:gwi}
\mathbb{E}_k\big\|g_{w}^i(k) \big\|^2 \leq& \mathbb{E}_k\left\| \frac{\partial L_{\mathcal{D}_i}( w^i_{k,N}, \phi_k)}{\partial w_k} \right\|^2= \mathbb{E}_k\left\|\prod_{m=0}^{N-1}(I - \alpha \nabla_w^2L_{\mathcal{S}_i}(w_{k,m}^i,\phi_k)) \nabla_{w} L_{\mathcal{D}_i} (w_{k,N}^i,\phi_k)\right\|^2 \nonumber
\\\leq &  \mathbb{E}_k(1+\alpha L)^{2N} M^2 = (1+\alpha L)^{2N} M^2.
\end{align}
Using an approach similar to~\cref{eq:gwi}, we have
\begin{align}\label{eq:gphi}
\mathbb{E}_k\big\|g_{\phi}^i(k) \big\|^2  \leq& 2\mathbb{E}_k\bigg\|\alpha \sum_{m=0}^{N-1}\nabla_\phi\nabla_w L_{\mathcal{S}_i}(w_{k,m}^i,\phi_k) \prod_{j=m+1}^{N-1}(I-\alpha\nabla_w^2L_{\mathcal{S}_i}(w_{k,j}^i,\phi_k))\nabla_w L_{\mathcal{D}_i}(w_{k,N}^i,\phi_k)\bigg\|^2 \nonumber
\\&+2\|\nabla_\phi L_{\mathcal{D}_i}(w_{k,N}^i,\phi_k)\|^2 \nonumber
\\ \leq& 2\alpha^2 L^2 M^2 \mathbb{E}_k\Big(\sum_{m=0}^{N-1} (1+\alpha L)^{N-1-m}\Big)^2 +2 M^2 \nonumber
\\< & 2M^2 (1+\alpha L)^{N}-1)^2 +2M^2 <  2M^2 (1+\alpha L)^{2N}.
\end{align}
Substituting~\cref{eq:gwi} and~\cref{eq:gphi} into~\cref{eq:finals}, we have
\begin{align}\label{eq:mid}
\mathbb{E}_k L^{meta}(w_{k+1},\phi_{k+1}) \leq& L^{meta}(w_k,\phi_k) - \left(\beta_w-\frac{L_w+L_\phi}{2}\beta_w^2\right)\left\| \frac{\partial L^{meta}(w_k,\phi_k)}{\partial w_k}\right\|^2 \nonumber
\\&+ \frac{(L_w+L_\phi)\beta_w^2}{2B}(1+\alpha L)^{2N}M^2-\big(\beta_\phi -L_\phi\beta_\phi^2 \big)\left\| \frac{\partial L^{meta}(w_{k},\phi_k)}{\partial \phi_k}  \right\|^2\nonumber
\\ & + \frac{2 L_\phi \beta_\phi^2}{B}  (1+\alpha L)^{2N}M^2.
\end{align}
Set  $\beta_w =\frac{1}{L_w+L_\phi}$ and $\beta_\phi =  \frac{1}{2L_\phi}$. Then, unconditioning on $w_k,\phi_k$ in~\cref{eq:mid}, we have
\begin{align*}
\mathbb{E} L^{meta}(w_{k+1},\phi_{k+1}) \leq& \mathbb{E}L^{meta}(w_k,\phi_k) - \frac{\beta_w}{2}\mathbb{E}\left\| \frac{\partial L^{meta}(w_k,\phi_k)}{\partial w_k}\right\|^2 + \frac{\beta_w}{2B}(1+\alpha L)^{2N}M^2\nonumber
\\ &-\frac{\beta_\phi}{2}\mathbb{E}\left\| \frac{\partial L^{meta}(w_{k},\phi_k)}{\partial \phi_k}  \right\|^2 + \frac{\beta_\phi}{B}  (1+\alpha L)^{2N}M^2.
\end{align*}
Telescoping the above equality over $k$ from $0$ to $K-1$ yields
\begin{align}\label{eq:conv}
\frac{\beta_w}{2}\frac{1}{K}&\sum_{k=0}^{K-1}\mathbb{E}\left\| \frac{\partial L^{meta}(w_k,\phi_k)}{\partial w_k}\right\|^2 + \frac{\beta_\phi}{2}\frac{1}{K}\sum_{k=0}^{K-1}\mathbb{E}\left\| \frac{\partial L^{meta}(w_{k},\phi_k)}{\partial \phi_k}  \right\|^2 \nonumber
\\& \leq  \frac{L^{meta}(w_0,\phi_0)-\min_{w,\phi}L^{meta}(w,\phi) }{K}+ \frac{\beta_w+2\beta_\phi}{2B}(1+\alpha L)^{2N}M^2.
\end{align}
Let $\Delta = L^{meta}(w_0,\phi_0)-\min_{w,\phi}L^{meta}(w,\phi)>0$ and let $\xi$ be chosen from $\{0,...,K-1\}$ uniformly at random. Then,~\cref{eq:conv} further yields
\begin{align*}
\mathbb{E}\left\| \frac{\partial L^{meta}(w_\xi,\phi_\xi)}{\partial w_\xi}\right\|^2  \leq& \frac{2\Delta (L_w+L_\phi)}{K} + \frac{1+\frac{L_w+L_\phi}{L_\phi}}{B}(1+\alpha L)^{2N}M^2 \nonumber
\\\mathbb{E}\left\| \frac{\partial L^{meta}(w_\xi,\phi_\xi)}{\partial \phi_\xi}\right\|^2  \leq &\frac{4\Delta L_\phi}{K} + \frac{2+\frac{2L_\phi}{L_w+L_\phi}}{B}(1+\alpha L)^{2N}M^2,
\end{align*}
which, in conjunction with the definitions of $L_w$ and $L_\phi$ in~\cref{eq:definitions} and using $\alpha<\mathcal{O}(\frac{1}{N})$, yields 
\begin{align}
\mathbb{E}\left\| \frac{\partial L^{meta}(w_\xi,\phi_\xi)}{\partial w_\xi}\right\|^2  \leq& \mathcal{O}\bigg(  \frac{N}{K} + \frac{N}{B}  \bigg), \nonumber
\\\mathbb{E}\left\| \frac{\partial L^{meta}(w_\xi,\phi_\xi)}{\partial \phi_\xi}\right\|^2  \leq &\mathcal{O}\bigg(  \frac{N}{K} + \frac{N}{B}  \bigg).
\end{align}
To achieve an $\epsilon$-stationary point, i.e., {\small $\mathbb{E}\left\| \frac{\partial L^{meta}(w,\phi)}{\partial w}\right\|^2 <\epsilon,\mathbb{E}\left\| \frac{\partial L^{meta}(w,\phi)}{\partial w}\right\|^2 <\epsilon$}, $K$ and $B$ need to be at most $\mathcal{O}(N\epsilon^{-2})$, which, in conjunction with the gradient forms in~\Cref{le:gd_form}, completes  the complexity results.


\end{document}